%% file: main.tex
\documentclass[twoside]{article}

\usepackage[accepted]{aistats2025}
\usepackage[round]{natbib}

\usepackage[utf8]{inputenc}
\usepackage{algorithm, algpseudocode}
\usepackage[final, babel]{microtype}
\usepackage{amsmath,amsthm,amssymb,amsfonts}
\usepackage{thmtools}
\usepackage{thm-restate}
\usepackage{hyperref}
\usepackage{graphicx}
\usepackage{caption}
\usepackage{booktabs}
\usepackage{float}
\usepackage{cite}
\usepackage{bm}
\usepackage{xcolor}
\usepackage{enumitem}
\usepackage{pifont}
\usepackage{makecell}

\include{notations}

\begin{document}

\twocolumn[

\aistatstitle{Steering No-Regret Agents in MFGs under Model Uncertainty}

\aistatsauthor{ Leo Widmer \And Jiawei Huang \And  Niao He }

\aistatsaddress{
    \texttt{lewidmer@student.ethz.ch, \{jiawei.huang, niao.he\}@inf.ethz.ch} \\
    Department of Computer Science\\
    ETH Zürich
}
]

\allowdisplaybreaks

\begin{abstract}
    Incentive design is a popular framework for guiding agents' learning dynamics towards desired outcomes by providing additional payments beyond intrinsic rewards. However, most existing works focus on a finite, small set of agents or assume complete knowledge of the game, limiting their applicability to real-world scenarios involving large populations and model uncertainty. To address this gap, we study the design of steering rewards in Mean-Field Games (MFGs) with density-independent transitions, where both the transition dynamics and intrinsic reward functions are unknown. This setting presents non-trivial challenges, as the mediator must incentivize the agents to explore for its model learning under uncertainty, while simultaneously steer them to converge to desired behaviors without incurring excessive incentive payments. Assuming agents exhibit no(-adaptive) regret behaviors, we contribute novel optimistic exploration algorithms. Theoretically, we establish sub-linear regret guarantees for the cumulative gaps between the agents' behaviors and the desired ones. In terms of the steering cost, we demonstrate that our total incentive payments incur only sub-linear excess, competing with a baseline steering strategy that stabilizes the target policy as an equilibrium. Our work presents an effective framework for steering agents behaviors in large-population systems under uncertainty.

\end{abstract}

\input{introduction}

\input{preliminaries}

\input{SteeringSetting}
\input{steering_strategies}

\input{incentive_design}

\input{discussion}
\section*{Acknowledgements}
This work is supported by Swiss National Science Foundation (SNSF) Project Funding No. 200021-207343 and SNSF Starting Grant.

\bibliography{refs}

\section*{Checklist}

 \begin{enumerate}

 \item For all models and algorithms presented, check if you include:
 \begin{enumerate}
   \item A clear description of the mathematical setting, assumptions, algorithm, and/or model. [Yes]
   \item An analysis of the properties and complexity (time, space, sample size) of any algorithm. [Yes]
   \item (Optional) Anonymized source code, with specification of all dependencies, including external libraries. [Not Applicable]
 \end{enumerate}

 \item For any theoretical claim, check if you include:
 \begin{enumerate}
   \item Statements of the full set of assumptions of all theoretical results. [Yes]
   \item Complete proofs of all theoretical results. [Yes]
   \item Clear explanations of any assumptions. [Yes]     
 \end{enumerate}

 \item For all figures and tables that present empirical results, check if you include:
 \begin{enumerate}
   \item The code, data, and instructions needed to reproduce the main experimental results (either in the supplemental material or as a URL). [Not Applicable]
   \item All the training details (e.g., data splits, hyperparameters, how they were chosen). [Not Applicable]
         \item A clear definition of the specific measure or statistics and error bars (e.g., with respect to the random seed after running experiments multiple times). [Not Applicable]
         \item A description of the computing infrastructure used. (e.g., type of GPUs, internal cluster, or cloud provider). [Not Applicable]
 \end{enumerate}

 \item If you are using existing assets (e.g., code, data, models) or curating/releasing new assets, check if you include:
 \begin{enumerate}
   \item Citations of the creator If your work uses existing assets. [Not Applicable]
   \item The license information of the assets, if applicable. [Not Applicable]
   \item New assets either in the supplemental material or as a URL, if applicable. [Not Applicable]
   \item Information about consent from data providers/curators. [Not Applicable]
   \item Discussion of sensible content if applicable, e.g., personally identifiable information or offensive content. [Not Applicable]
 \end{enumerate}

 \item If you used crowdsourcing or conducted research with human subjects, check if you include:
 \begin{enumerate}
   \item The full text of instructions given to participants and screenshots. [Not Applicable]
   \item Descriptions of potential participant risks, with links to Institutional Review Board (IRB) approvals if applicable. [Not Applicable]
   \item The estimated hourly wage paid to participants and the total amount spent on participant compensation. [Not Applicable]
 \end{enumerate}

 \end{enumerate}

\newpage
\onecolumn
\tableofcontents
\newpage
\appendix
\input{Appendix/supplement}

\end{document}

%% file: notations.tex
\newcommand{\sad}{\mu} %
\newcommand{\sd}{\mu} %

\newcommand{\be}{\mathbf{e}}

\newcommand{\RR}{\mathbb{R}}
\newcommand{\EE}{\mathbb{E}}
\newcommand{\PP}{\mathbb{P}}
\newcommand{\II}{\mathbb{I}}
\newcommand{\NN}{\mathbb{N}}

\newcommand{\cS}{\mathcal{S}}
\newcommand{\cA}{\mathcal{A}}
\newcommand{\cP}{\mathcal{P}}
\newcommand{\cR}{\mathcal{R}}

\newcommand{\cF}{\mathcal{F}}
\newcommand{\cE}{\mathcal{E}}
\newcommand{\cX}{\mathcal{X}}
\newcommand{\cO}{\mathcal{O}}
\newcommand{\cU}{\mathcal{U}}

\DeclareMathOperator*{\argmax}{arg\,max}
\DeclareMathOperator*{\argmin}{arg\,min}

\DeclareMathOperator*{\reg}{Reg}
\DeclareMathOperator*{\adareg}{AdaReg}

\renewcommand{\epsilon}{\ensuremath\varepsilon}

\DeclareSymbolFont{oldletters}{OML}{cmm}{m}{it}

\newtheorem{theorem}{Theorem}[section]
\newtheorem{lemma}[theorem]{Lemma}

\newtheorem{proposition}[theorem]{Proposition}

\theoremstyle{definition}
\newtheorem{definition}[theorem]{Definition}

\newtheorem{assumption}{Assumption}

\newtheorem{remark}[definition]{Remark}

\newcommand{\bmu}{\bar{\mu}}

\newcommand{\bpi}{\bar{\pi}}

\newcommand{\nz}{\text{nz}}
\newcommand{\z}{\text{z}}

\newcommand{\tcO}{\tilde{\cO}}

%% file: introduction.tex
\section{INTRODUCTION}

Mean-Field Games (MFGs) \citep{huang2006large,lasry2007mean} are a widely-used and powerful framework to model the competition and cooperation of large population systems involving symmetric and interchangeable agents.
MFGs effectively capture the dynamics of many real-world scenarios, such as macro-economic models \citep{steinbacher_advances_2021}, road traffic systems \citep{chen2010trafficsystems}, autonomous vehicle systems \citep{dinneweth_multi-agent_2022} and auctions \citep{iyer2014mean}, and it has been successfully applied in those domains \citep{gomes_socio-economic_2014, cabannes_solving_2021, achdou_mean_2019, guo_learning_2021}.
Similar to the finite-agent systems \citep{Roughgarden2007-ts}, MFGs with self-interested agents may lead to undesirable collective behaviors.
Typically, the agents' learning dynamics may converge to equilibria where all the participants are worse off compared to other possible outcomes \citep{guo_mesob_2023}.

To address this dilemma, the field of incentive design explores methods to guide agents towards more favorable behaviors by modifying the reward structure. A widely-studied formulation, known as the \emph{steering problem} \citep{zhang_steering_2024,canyakmaz_steering_2024,huang_learning_2024}, assumes the presence of a mediator (incentive designer) outside the game, who can influence the agents' learning dynamics by providing additional steering rewards.
However, previous research on steering mainly focuses on either Extensive-Form Games \citep{zhang_steering_2024} or Markov Games \citep{canyakmaz_steering_2024,huang_learning_2024} with a limited number of agents.
The methods developed for those small-scale settings become intractable when applied to large-population scenarios, as the number of agents increases, which is known as the curse of multi-agency.

To address this gap, in this work, we study incentive design in Mean-Field Games (MFGs).
More concretely, we focus on the finite-horizon MFGs with density-independent transitions, a standard model in literature \citep{huang2006large,lasry2007mean,perolat_scaling_2021}.
In the steering problem setup, we play the role of the mediator, with access to a utility function dependent on the collective behavior of the agents through the population density.
During the interactions with the mediator, the agents are continuously learning and adapting. We assume the agents are self-interested no-adaptive-regret learners \citep{AdaptiveRegret}; similar no-regret assumptions have been widely adopted in previous literature \citep{camara_mechanisms_2020, ge2024principledsuperhumanaimultiplayer}.
Following previous works \citep{zhang_steering_2024,huang_learning_2024}, our primary goal is to design steering rewards that guide the agents towards desired policies (i.e., minimize the \emph{steering gap}), such that the resulting behaviors maximize the utility function. Meanwhile, the incentives paid by the mediator to the agents, referred to as the \emph{steering cost}, should remain low.

In practice, the mediator usually lacks knowledge of the transition dynamics and the intrinsic reward functions of the MFGs.
Therefore, in this work, we focus on the design of steering strategies without prior knowledge of the game model.
The model uncertainty makes the steering problem much more challenging, and requires the mediator to strategically balance the exploration and exploitation.
Typically, without knowledge of the MFG model, the mediator does not know which agents' behaviors (population densities) are feasible, let alone how to maximize the utility function.
Therefore, the mediator needs not only to steer the agents to explore the MFG for its own learning, but also ensure the agents converge to desired outcomes, while keeping the accumulative incentive payments affordable.
In summary, the key question we would like to address is:
\begin{center}
\emph{
    How can we design effective steering strategies for no-regret agents in MFGs under model uncertainty?
}
\end{center}

\textbf{Main contributions}\quad
We address the above open question by proposing novel exploration algorithms with provable guarantees.
We highlight our main contributions in the following. A summary of the main theorems in this paper can be found in Appx.~\ref{appx:summary}.

\begin{itemize}[leftmargin=*]
    \item Firstly, in Sec.~\ref{sec:steering_formulation}, we contribute the first formulation for steering in mean-field games, with details about the problem setting and learning objectives.

    \item Secondly, as preparation, in Sec.~\ref{chapter: steering}, we investigate how to steer the agents to a given density or policy.
    Notably, in Sec.~\ref{sec: policy incentivization}, we propose a novel steering strategy, which can guide the no-adaptive-regret agents towards any target policy without prior knowledge of the model.
    This method serves as the key ingredient of our steering algorithms in the following sections.
    
    \item Thirdly, in Sec.~\ref{sec: incentive design with no original reward} and Sec.~\ref{sec: incentive design with unknown reward}, we investigate strategic exploration methods for steering agents in MFGs under uncertainty.
    In Sec.~\ref{sec: incentive design with no original reward}, we start with the setting where the intrinsic reward is zero, and propose an optimism-based exploration algorithm, which guarantees that both the cumulative steering gap and cost only have sub-linear growth.
    Furthermore, in Sec.~\ref{sec: incentive design with unknown reward}, we extend our methods to the setting with non-zero and unknown intrinsic reward by integrating a pessimism-based reward estimation strategy. We establish sub-linear regret in steering gap, and show that the total steering cost is only sub-linearly worse, compared to a baseline strategy that stabilizes the target policy as an equilibrium by offsetting differences in intrinsic rewards.
\end{itemize}

\subsection{Closely Related Work}\label{sec:related_work}
Due to the limit of space, we only discuss closely related works here and defer the others to Appendix~\ref{appx:related_works}.

\textbf{Incentive Design in Multi-Agent Systems}\quad
The problem of incentive design broadly refers to the design of mechanisms for shaping the behavior of autonomous agents \citep{ehtamo2002recent, ratliff2019perspective}.
A recently popular framework for incentive design is known as the steering problem \citep{zhang_steering_2024,canyakmaz_steering_2024, huang_learning_2024}, which considers a repeated interaction between a mediator and learning agents.
All of them focuses on small-scale problems (e.g., Markov Games or Extensive-Form Games) and their proposed methods become intractable when extending to large-population setting, including MFGs.
Besides, \citet{canyakmaz_steering_2024, huang_learning_2024} consider the agents' learning dynamics to be memoryless, which is different from our no-regret assumptions.

Another related direction is contract design\footnote{To save space, we defer to Appx.~\ref{appx:related_works} more elaboration of the comparisons between our steering framework and contract design setting.} \citep{dellavigna2004contract}, which studies the interactions between a principal and agents when the two parties transact in the presence of private information.
The fundamental question is how the principal should design the incentives for the agents to maximize its own utility after deducting the payments to agents.
However, most of literature study the single-agent setting 
\citep{zhu2022sample, ho2014adaptive, scheid2024incentivized}, or focus on the computational aspects without addressing exploration under uncertainty \citep{dutting2023multi,castiglioni2023multi}.
\citep{carmona2021finite, elie2019tale} study contract design in the MFGs setting, but none of them consider model uncertainty.
Moreover, a common assumption in those works is that the agents always do the best response (or take equilibrium policies) to the principal's intervention, which is much stronger than our no-adaptive-regret assumption.

Besides, \citet{sanjari_incentive_2024} consider incentive design in a large-population setting, but they study the Stackelberg games with one leader and a large number of followers, which differs quite substantially to ours.
\citet{fu2018meanfieldleaderfollowergames} consider mean-field leader-follower games, however they assume knowledge of the dynamics, while we consider the steering problem without this knowledge.
Moreover, they study the dynamics where the agents cooperate together and optimally respond to the leader’s control signal. In contrast, we consider decentralized and self-interested agents with no-regret behaviors in maximizing individual interests.

%% file: preliminaries.tex
\section{PRELIMINARIES}\label{chapter: preliminaries}\label{sec: MFG}\label{sec: eluder dimension}

\textbf{Mean-Field Games}\quad
We consider the MFG setting with a finite yet extremely large number of agents, each of which acts independently. In line with \citet{subramanian2022decentralized}, we refer to this setting as ``Decentralized-MFGs'', although their model allows diversity in action space and reward functions for agents.
\begin{definition}
    A \emph{Finite-Horizon Decentralized MFG} is defined by a tuple $M=(N,\cS,\cA, H,\PP_M,r_M,\sd_1)$, given the number of agents $N$; state and action spaces $\cS,\cA$ with sizes $S$ and $A$; horizon length $H$; initial state distribution $\sd_1\in\Delta_\cS$.
    $\PP_M:=\{\PP_{M,h}\}_{h=1}^H$ with $\PP_{M,h}:\cS\times\cA\to\Delta_\cS$ and $r_M:=\{r_{M,h}\}_{h=1}^H$ with $r_{M,h}:\cS\times\cA\times\Delta_{\cS\times\cA}\to[0,r_{\max}]$ denote the transition and reward function, respectively.
\end{definition}
In this paper, we focus on density-independent transition function, a common assumption in previous literature \citep{huang2006large,lasry2007mean,perolat_scaling_2021}.
For the reward function, we consider the general setup, where the rewards depend on the state-action density \citep{guo_learning_2021}.

We only focus on non-stationary Markovian policies, denoted by $\Pi:=\{\pi:=\{\pi_h\}_{h\in[H]}|\pi_h:\cS\rightarrow\Delta(\cA)\}$.
Given a model $M$, considering an agent taking policy $\pi\in\Pi$, we use $\sad^{\pi}_M:=\{\sad^{\pi}_{M,h}\}_{h=1}^H$
to denote its state-action density for each step $h\in[H]$.
Starting with $\sad^{\pi}_{M,1}(s,a)=\sd_1(s)\pi_1(a|s)$, for $1\leq h\leq H$, we have:
\begin{align*}
    \sad^{\pi}_{M,h+1}(s,a)=\pi_{h+1}(a|s)\sum_{s',a'}\PP_{M,h}(s|s',a')\sad^{\pi}_{M,h}(s',a').
\end{align*}

When $N$ agents take policies $\pi^1,...\pi^N\in\Pi$, respectively, the trajectory of agent $n\in[N]$ is specified by:
\begin{align}
    \textstyle s_1^n \sim \mu_1,~ \forall h\geq 1,~ &a_h^n \sim \pi^n_h(\cdot|s_h^n),
    ~s_{h+1}^n\sim\PP_{M,h}(\cdot|s_h^n,a_h^n),\nonumber
    \\ &r_h^n\gets r_{M,h}(s_h^n,a_h^n,\bar\sad_{M,h}).\label{eq:traj_generation}
\end{align}
where we use $\bar\sad_{M,h}=\frac1N\sum_{n=1}^N\sad^{\pi^{n}}_{M,h}$ to denote the \emph{population density} at step $h$.
We also assume $r_{M}$ is Lipschitz in the density, which is standard in previous works \citep{guo_learning_2021,yardim_policy_2022}.

\textbf{Other Notational Convention}\quad
For convenience, we implicitly treat $\sad^{\pi}_M$ as a vector in $\RR^{HSA}$ concatenated by $\{\mu_{M,h}^{\pi}\}_{h\in[H]}$.
We denote $\Psi_M:=\{\sad^{\pi}_M:\pi\in\Pi\}\subseteq\Delta_{\cS\times\cA}^H$ to be the set of all feasible state-action densities given $M$.
Note that $\Psi_M$ is a convex set (see Lem.~\ref{lem: state-action distribution set}), which implies $\bmu_M \in \Psi_M$.
If it is not necessary to distinguish what model $M$ we use, we omit it in the sub-scriptions, for example, $\sad^\pi/\Psi$ instead of $\sad^{\pi}_M/\Psi_M$.
We also omit $h$ in $s_h,a_h$ if it is clear from the context.
With slight abuse of notation, given a population density $\bmu:=\{\bmu_h\}_{h=1}^H\in\Delta^H_{\cS\times\cA}$ and a reward function $r$, we use $r(\bmu) \in \RR^{HSA}$ to denote the reward vector where $(r(\bmu))_{h,s,a}=r_h(s,a,\bmu_h)$.
In this way, given an arbitrary agent $n\in[N]$ taking policy $\pi$, its expected total return conditioning on population density $\bmu$ can be written as: $\EE_{\pi^n}[\sum_{h=1}^H r_h(s_h^n,a_h^n,\bmu_h)] = \langle r(\bmu), \mu^{\pi^n} \rangle$.

Given that this paper considers learning under uncertainty, we use $M^*$ to denote the true hidden mean-field model with transition $\PP^*$ and intrinsic reward $r^*$, in order to distinguish it from the estimated ones.

Besides, given a population density $\bmu$ in a model $M$, we will use $\bpi$ to denote the policy, which induces the population density (i.e., $\mu^{\bpi} = \bmu$), defined by: $\bpi_h(\cdot|s) := \bmu_{h}(s,\cdot) / \bmu_{h}(s)$ (or $\bpi_h(\cdot|s) = 1/A$ if $\bmu_{h}(\cdot)=0$).

\textbf{Reward Function Approximation and Eluder Dimension}\quad
In this paper, we consider the setting where the true intrinsic reward, denoted by $r^*$, is unknown. Note that the reward function depends on not only the state and action but also the density, which belongs to a high-dimensional continuous space.
Therefore, we consider function approximation for reward estimation with the standard realizability assumption.
\begin{assumption}\label{assump:realizability}
    A reward function class $\cR$ is available, s.t. (i) $\forall r\in\cR$, $\forall h,~r_h(\cdot,\cdot,\cdot)\in[0,r_{\max}]$; (ii) $r^* \in \cR$.
\end{assumption}
In the function approximation setting, the fundamental sample efficiency is closely related to the complexity of the function class.
We follow previous works \citep{Eluder,huang_model-based_2024} and utilize the Eluder Dimension as the complexity measure of the function class.
Intuitively, the Eluder Dimension is defined to be the length of the longest ``independent'' sequence, such that each element in the sequence ``reveals'' some new information about the function class comparing with previous ones.
\begin{definition}[$\epsilon$-independent sequence]
    Given a domain $\cX$ and a class of functions $\cF$ defined on $\cX$, we say $x\in\cX$ is $\epsilon$-independent on $\{x_1,...,x_J\}\subseteq\cX$ if there exists $f,\tilde f\in\cF$, such that $\sum_{j=1}^J(f(x_j)-\tilde f(x_j))^2\leq\epsilon^2$, but $|f(x)-\tilde f(x)|>\epsilon$. 
\end{definition}
\begin{definition}[Eluder Dimension]
    Given a mean-field reward function class $\cR$ and domain $\cX := [H]\times\cS\times\cA\times\Delta_{\cS\times\cA}$, the Eluder Dimension of $\cR$, denoted by $\dim_E(\cR,\epsilon)$, is defined to be the length of the longest sequence $\{x^j\}_{j=1}^J$, such that, for any $i\in[J]$, $x^i$ is $\epsilon$-independent w.r.t.\ $\{x^j\}_{j=1}^{i-1}$.
\end{definition}

%% file: SteeringSetting.tex
\section{THE STEERING PROBLEM FORMULATION FOR MFGS}\label{sec:steering_formulation}
In this section, we introduce our steering setup.
In Sec.~\ref{sec: intro to incentive design}, we first provide our formulation for steering protocol. Then, in Sec.~\ref{sec: behavioral assumption on agents} we discuss our assumptions on agent's behavior. After that, we introduce the learning objectives and other setups in Sec.~\ref{sec:performance metrics} and~\ref{sec:objectives}.

\subsection{Agent-Mediator Interaction Protocol}\label{sec: intro to incentive design}
We consider a repeated game setup, and summarize the interaction procedure between agents and the mediator in Procedure~\ref{alg: agent-mediator interaction}. 
In each iteration $t\in[T]$, the mediator first selects a steering reward function\footnote{We will use capital $R$ to denote the steering reward to distinguish with intrinsic reward $r$.} $R^{t}$, which is a mapping from the density space to the \emph{non-negative}\footnote{The non-negativity of the steering reward is known as limited liability \citep{innes1990limited}, which is standard in previous works \citep{zhang_steering_2024,huang_learning_2024}} reward vector space, upper bounded by $R_{\max}$.
Besides, each agent computes a policy and plays the game.
The agents' policies result in a population density $\bmu^{t} := \frac{1}{N}\sum_{n=1}^N \mu^{\pi^{n,t}}$, by which the mediator realizes the steering reward $R^{t}(\bmu^{t})$.
Then, each agent $n\in[N]$ receives payments from the mediator equal to the expected return induced by the steering reward and the agent's policy, i.e., $\langle R^{t}(\bmu^{t}), \mu^{\pi^{n,t}} \rangle$.
We highlight here that in our setup, at each iteration $t$, the mediator designs the steering reward function $R^{t}$ without the knowledge of the agents' policies $\pi^{n,t}$, and we do not restrict whether the agents can observe $R^{t}$ or not before they make decisions.
Furthermore, the agents can either independently compute their policies or collaborate.
In the next section, we will characterize our assumptions on the agents' behaviors with more details.

\begin{algorithm}
\makeatletter
\renewcommand{\ALG@name}{Procedure}
\makeatother
\begin{algorithmic}[1]
    \For{$t=1,...,T$}
        \State Mediator chooses $R^{t}:\Delta_{\cS\times\cA}^H\to[0,R_{\max}]^{HSA}$.
        \State Each agent $n\in[N]$ computes policy $\pi^{n,t}\in\Pi$, resulting in the population density $\bmu^{t}$, and gets payment $\langle R^{t}(\bmu^{t}),\sad^{\pi^{n,t}}\rangle$ from the mediator.
        \State Mediator observes $\bmu^{t}$ and a trajectory $\{(s_h^{n,t},a_h^{n,t},r_h^{n,t} + \xi_h^{t})\}_{h\in[H]}$ generated by $\pi^{n,t}$ following Eq.~\eqref{eq:traj_generation}, where $n\sim\text{Uniform}\{1,2,...,N\}$.
    \EndFor
\caption{Agent-Mediator Interaction Protocol}\label{alg: agent-mediator interaction}
\end{algorithmic}
\end{algorithm}

At the end of each iteration, the mediator can observe a trajectory sampled from a random agent with noisy reward samples.
We assume noises $\xi_h^{t}$ are i.i.d.\ $\sigma$-sub-Gaussian random variables with zero mean.
We also assume the mediator has access to the population density, which is necessary to estimate the unknown intrinsic reward function from samples.

\subsection{Behavioral Assumptions on Agents}\label{sec: behavioral assumption on agents}
We first introduce our no-adaptive regret assumption and its implication, and then make some justification.
\begin{assumption}[No-Adaptive Regret Behavior]\label{ass: adaptive regret}
    In Procedure~\ref{alg: agent-mediator interaction}, 
    the adaptive regret for each agent $\forall n \in [N]$, which is defined below, can be upper bounded by some term $\adareg(T)=(r_{\max} + R_{\max}) \cdot o(T)$:
    \begin{align}
        \max_{\substack{1\leq a<b\leq T \\\sad\in\Psi_{M^*}}} & \sum_{t=a}^b\langle r^*(\bmu^{t}) + R^{t}(\bar{\mu}^{t}),\sad-\sad^{\pi^{n,t}}\rangle \label{eq:no_adaptive_regret},
    \end{align}
    where $(r_{\max} + R_{\max})$ is the normalization term. In Appx.~\ref{appx : boundedness of steering rewards} we show that for all the steering rewards that we deploy in this paper we have $R_{\max}=\cO(1+r_{\max})$.
\end{assumption}

\textbf{Justification for Assump.~\ref{ass: adaptive regret}}\quad
We remark that it is common to consider agents exhibiting no-regret behaviors in previous literature \citep{deng_strategizing_2019, zhang_steering_2024, brown2024learning}.
Most of these literature assume no-external regret (directly assigning $a=1$ and $b=T$ in Eq.~\eqref{eq:no_adaptive_regret}), which is weaker than our no-adaptive-regret assumption.
However, similar stronger assumptions, such as no-dynamic-regret learners, have also been considered in some studies \citep{ge2024principledsuperhumanaimultiplayer}.
Moreover, our no-adaptive-regret assumption is standard when interpreted through the online linear optimization perspective \citep{AdaptiveRegret,hazan_introduction_2023}, where in each iteration, each agent picks a density from the convex set $\Psi_{M^*}$ and receives potentially adversarial feedback $R^{t}(\bmu^{t})$.
Then, Assump.~\ref{ass: adaptive regret} aligns with the standard no-adaptive regret guarantees in online linear optimization setting, and there are very simple algorithms (e.g., Online Gradient Descent) achieving $\adareg = \tilde{O}(\sqrt{T})$.
We defer more detailed discussion to Appx.~\ref{appx:concrete_example}.

Under Assump.~\ref{ass: adaptive regret}, we have the following property, which suggests the collective population will also exhibit no-regret behaviors. This is a useful property we will leverage in algorithm design.
\begin{restatable}[No-Adaptive-Regret Population Behavior]{proposition}{PropNoRegretPopulation}\label{prop:no_regret_population}
    Under Assump.~\ref{ass: adaptive regret}, we have:
    \begin{align*}
        \max_{1\leq a<b\leq T,\sad\in\Psi_{M^*}}
        &\sum_{t=a}^b \langle r^*(\bmu^{t}) + R^{t}(\bmu^{t}), \mu - \bmu^{t}\rangle \\
        &\leq \adareg(T) = (r_{\max} + R_{\max}) \cdot o(T).
    \end{align*}
\end{restatable}

\subsection{Performance Metrics}\label{sec:performance metrics}
Inspired by the previous works \citep{zhang_steering_2024,huang_learning_2024}, we evaluate the steering algorithm from two aspects: the steering gap and the steering cost. We provide the concrete definition in our MFGs setup as follows.

\textbf{The Steering Gap}\quad
Intuitively, the steering gap measures the difference between the desired outcomes and the agents' behavior under the mediator's guidance.
In this paper, we assume the mediator is given a utility function $U:\Delta_{\cS\times\cA}^H\to\RR$ assigning each population density a utility value.
The only assumption we make for it is about the Lipschitz continuity:
\begin{assumption}[Lipschitz Utility Function]\label{assump:Lipschitz}
    $
        \forall \mu,\mu'\in\Delta^H_{\cS\times\cA},~ |U(\mu) - U(\mu')| \leq L_U\|\mu - \mu'\|_1.
    $
\end{assumption}
The steering gap up to step $T$ is defined by:
\begin{align*}
    \textstyle \Delta_T(\{\bmu^{t}\}_{t=1}^T) := \max_{\pi^*\in\Pi} \sum_{t=1}^T U(\mu^{\pi^*}) - U(\bmu^{t})
\end{align*}
Here $U(\bmu^{t})$ represents the utility paid to the mediator at each iteration $t\in[T]$, induced by the population density $\bmu^{t}$.
Note that we consider the best density maximizing utility function as the comparator.
This can be interpreted as the best population density if all the agents are restricted to take the same policy, and finding the best shared policy is a standard objective in previous MFGs literature.

\textbf{The Steering Cost}\quad
The motivation for introducing a steering cost is that the agents will not accept the mediator's guidance for free. A common measure of the cost is the expected total return associated with the reward received by the agents.
Formally, suppose at iteration $t$, the mediator computes a steering reward function $R^{t}$, and the $N$ agents select policies $\pi^{1,t},\pi^{2,t},...,\pi^{N,t} \in \Pi$, which induce a population density $\bmu^{t} := \frac{1}{N}\sum_{n=1}^N \mu^{\pi^{n,t}}$, then the steering cost is defined to be the average payments to the agents:
$C(\bmu^{t},R^{t}) := \langle R^{t}(\bmu^{t}), \bmu^{t} \rangle =\frac{1}{N}\sum_{n=1}^N \langle R^{t}(\bmu^{t}), \mu^{\pi^{n,t}} \rangle.$
We will use 
$$
\textstyle C_T(\{\bmu^{t},R^{t}\}_{t=1}^T):= \sum_{t=1}^T C(\bmu^{t},R^{t})
$$ 
to denote the accumulative steering gap.
Note that the steering rewards are non-negative, the steering cost effectively reflects the strength of the steering signal.

\subsection{Two Steering Scenarios and Objectives}\label{sec:objectives}
In this paper, we consider the case when the mediator does not know the true transition and reward functions of $M^*$.
However, to make it easy for reader to understand our algorithm design and technique contributions, we will start with a special case, where the agents do not have intrinsic rewards, i.e., $r^* = 0$.

\textbf{Scenario 1: No Intrinsic Reward}\quad
The goal of this setting to find an incentive design algorithm producing a sequence of $R^{t}$ such that both the steering gap and the steering cost are sub-linear:
$
    \Delta_T(\{\bmu^{t}\}_{t=1}^T) = o(T),~ C_T(\{\bmu^{t},R^{t}\}_{t=1}^T) = o(T).
$

The motivation for the sub-linear guarantee here is that it implies the average utility converges to the maximum and the average steering cost vanishes.
This implies that the incentive design strategies fulfilling these guarantees will eventually pay off as a long-term investment.
In Sec.~\ref{sec: incentive design with no original reward}, we analyze this case and provide algorithms achieving our objective.

\textbf{Scenario 2: Non-Zero Intrinsic Reward}\quad
In Sec.~\ref{sec: incentive design with unknown reward}, we study the complete setting where the agents' original reward is non-zero and unknown.
In this case, the mediator additionally has to estimate the reward function from observed noisy samples and steer the agents based on that.
Similarly, we expect sub-linear steering gap $\Delta_T(\{\bmu^{t}\}_{t=1}^T) = o(T)$, while for steering cost, we manually choose the ``sandboxing reward'' as the comparator:
\begin{align*}
C_T(\{\bmu^{t},R^{t} - \underbrace{(r_{\max}\cdot \mathbf{1} - r^*)}_{\text{sandboxing reward}}\}_{t=1}^T) = o(T). 
\end{align*}

Here we use $\mathbf{1}$ to denote the all-ones vector.
Intuitively, because the intrinsic rewards are non-zero, if the desired behavior $\mu^{\pi^*}$ is not an equilibrium induced by $r^*$, the mediator has to maintain a non-zero steering rewards to avoid the agents deviating from $\pi^*$, so we can not expect the average steering cost to vanish to 0 as in Scenario 1.
Therefore, we consider the sandboxing reward as a baseline comparator, which mitigates differences in the intrinsic rewards, so that $\pi^*$ would be a ``stable equilibrium'' even if the additional steering reward $R^{t} - (r_{\max} \cdot \mathbf{1} - r^*(\bmu^{t}))$ vanishes to zero.
Though, we admit that other choices of sandboxing terms may result in lower steering cost, or one can consider optimizing utility and steering cost together.
We leave those interesting directions for the future work.

%% file: steering_strategies.tex
\section{STEERING TOWARDS A FIXED TARGET}\label{chapter: steering}

In this section, we focus on how to design rewards to guide the agents to a target population density or policy, which serves as preparation steps for the following sections.
For convenience, we assume the agents' intrinsic rewards are zero and ignore them.

\subsection{Warm-Up: Steering in a Known Model}\label{sec: steering in known model}
We start with the case when the MFG model is known.
In this case, we can also compute $\Psi_{M^*}$ and find the best $\sad^*=\argmax_{\sad\in\Psi_{M^*}}U(\sad)$. 
If we want to steer the population to $\sad^*$, one steering reward choice is $R(\sad)=\bm1\Vert\sad^*-\sad\Vert_\infty + \sad^*-\sad$, where the first shift term is to ensure the non-negativity.
The key motivation for our choice is that $\langle R(\sad),\sad^*-\sad\rangle=\Vert\sad-\sad^*\Vert_2^2$.
As a result, if we consider the accumulative performance, we have the following theorem.
\begin{restatable}{theorem}{ThmUtilityRegretKnownModel}\label{thm: utility regret known model}
    If $M^*$ is known and $r^*$ is zero everywhere, under Asump.~\ref{ass: adaptive regret} and~\ref{assump:Lipschitz}, 
    by choosing the steering reward $\forall~t\in[T]$, $R^{t}(\sad)=\sad^*-\sad+\bm1\Vert\sad^*-\sad\Vert_\infty$, for any $\sad\in\Delta_{\cS\times\cA}^H$, we have: 
    \begin{align*}
    \Delta_T(\{\bmu^t_{M^*}\}_{t=1}^T)\leq L_U\sqrt{HSAT\adareg(T)} = o(T)\\
    C_T(\{\bmu^{t}_{M^*},R^t\}_{t=1}^T)\leq2H\sqrt{T\adareg(T)} = o(T).
    \end{align*}
\end{restatable}
The bound above for the known model setting, although may not be tight, can serve as a benchmark for the more challenging unknown model settings. We will see that the bounds in Theorems~\ref{thm: incentive designer regret}~and~\ref{thm: incentive designer regret unknown reward} (for the unknown model setting) are not much worse than the bound of Theorem~\ref{thm: utility regret known model}.

\subsection{Steering towards a Target Policy in an Unknown Model}\label{sec: policy incentivization}
Without the knowledge of transition function $\PP^*$, the steering becomes challenging, because we can no longer compute $\Psi_{M^*}$ or identify whether a given density (e.g. $\argmax_{\mu\in\Delta_{\cS\times\cA}^H}U(\mu)$) can actually be achieved by the agents.
Therefore, we shift our focus to the policy space.
Interestingly, we reveal that, it is possible to steer the agents to any target policy $\pi\in\Pi$, even without the knowledge of $\PP^*$ or $\Psi_{M^*}$.
Our key observation is the following lemma, which suggests an upper bound to control the difference between the population density and the density regarding the target policy.

\begin{restatable}{lemma}{LemDensityDiffUBSpecial}\label{lem:density_diff_upper_bound_specialized}
    Given any $M$ and target $\pi \in \Pi$, suppose the agents induce population density $\bmu_M$ in $M$, then:
    \begin{align}
        \Vert\bmu_{M}-\sad^{\pi}_{M}\Vert_1\leq& H\sum_{h,s}\bmu_{M,h}(s)\Vert\bpi_h(\cdot|s)-\pi_h(\cdot|s)\Vert_1,\nonumber\\
        \text{with }&\bmu_{M,h}(s):=\sum_a\bmu_{M,h}(s,a) \label{eq:density_diff}
    \end{align}
\end{restatable}
This motivates us to design a steering reward function that penalizes the RHS of Eq.~\eqref{eq:density_diff}, which is actually doable \emph{without the knowledge of model}.
Given a policy $\pi$, we define matrix $W^\pi \in \RR^{SAH\times SAH}$ to be the block diagonal of $W^\pi_{h,s_h}$ for all $h\in[H]$ and $s_h\in\cS$, where
\begin{align}\label{eq: policy reward matrix}
\begin{split}
    W^\pi_{h,s}&:=%
    \left[
        \begin{matrix}
            \pi_h(a_1|s)&\hdots&\pi_h(a_1|s)\\
            \vdots&&\vdots\\
            \pi_h(a_A|s)&\hdots&\pi_h(a_A|s)
        \end{matrix}
    \right] \in \RR^{A\times A},
\end{split}
\end{align}
Now, consider the steering reward function:
\begin{align}\label{eq: policy incentive reward}
    \forall \mu\in\Delta^H_{\cS\times\cA},~R_\pi(\mu):=-\sad^\top(W^\pi-I)^\top(W^\pi-I).
\end{align}
where $I \in \RR^{SAH\times SAH}$ is the identity matrix.
We can verify that, for any possible population density $\bmu^{t}$ occurs at step $t$, we have
\begin{align}
    \langle R_\pi(\bmu^{t}),\mu^\pi - \bmu^{t}\rangle=\Vert(W^\pi-I)\bmu^{t}\Vert_2^2\nonumber\\
    =\sum_{h,s,a} (\bmu_h^{t}(s))^2|\pi_h(a|s) - \bar\pi_h^{t}(a|s)|^2.\label{eq:connection}
\end{align}
Recall that $\bpi_h^{t}$ denotes the policy induced by population density (see definition in Sec.~\ref{chapter: preliminaries}).
Here in the first equality, we use the fact that, for any $\mu$, $\langle R_\pi(\mu), \mu^\pi \rangle = 0$ since $(W^\pi-I)\sad^\pi=0$.
Eq.~\eqref{eq:connection} above is important in that it connects the one step regret (LHS) with the gap between the population density and target density (RHS through Lemma~\ref{lem:density_diff_upper_bound_specialized}).

Combining with Prop.~\ref{prop:no_regret_population}, if all the agents are no-regret learners, and we steer the agents with the same steering reward $R_\pi$ for $T$ steps, we should expect $\bmu^{t}$ to converge to $\mu^\pi$, which we summarize to the following theorem. 
This result provides important insights for our incentive design algorithm in Section~\ref{sec: incentive design with no original reward}.
\begin{restatable}{theorem}{ThmRegretPolicyReward}\label{thm: regret of policy reward}
    Let $\pi^*=\argmax_{\pi}U(\mu^{\pi})$ and $R^t(\mu)=R_{\pi^*}(\mu)+\|R_{\pi^*}(\mu)\|_\infty\bm1$ for all $t$. Under Assump.~\ref{ass: adaptive regret},
    \begin{align*}
        \Delta_T(\{\bmu^t_{M^*}\}_{t=1}^T)\leq L_U\sqrt{H^3SAT\adareg(T)} = o(T)\\
        C_T(\{\bmu^{t}_{M^*},R^t\}_{t=1}^T)\leq4H\sqrt{T\adareg(T)} = o(T).
    \end{align*}
\end{restatable}

%% file: incentive_design.tex
\section{STEERING WITH NO INTRINSIC REWARD}\label{sec: incentive design with no original reward}

In this section, we study the \textbf{Scenario 1} introduced in Sec.~\ref{sec:objectives}, where the transition function $\PP^*$ is unknown and the original reward $r^*$ is zero, so the steering rewards are the only incentives for the agents.
The main challenge in this setting is that, without the knowledge of $\PP^*$, we can not determine the feasible density set $\Psi_{M^*}$ and the maximizer of the utility function.
Therefore, we have to design a steering strategy to incentivize the agents to explore for the mediator to estimate $\PP^*$, while balancing the exploration-exploitation trade-off to ensure sub-linear steering gap and cost.

Our main contribution is an optimism-based exploration algorithm in Alg.~\ref{alg: incentive design}, which provably addresses the above challenges and achieves our objectives.
The algorithm is built based on the techniques we developed in Sec.~\ref{sec: policy incentivization}, which allows us to steer the agents to any target policy without the knowledge of model.
Next, we introduce the key components in algorithm design.
\begin{algorithm}
\caption{Steering reward design for Scenario 1}
\label{alg: incentive design}
\begin{algorithmic}[1]
    \State Initialize $\cP^{1}:=$ set of all possible transition functions, $\pi_*^{1}$ (arbitrarily), $k=1,T_0=0$.
    \For {$t=1,...,T$}
        \Statex\Comment Recall $R_{\pi_*^{k}}$ as defined in Eq.~\eqref{eq: policy incentive reward}
        \State Compute steering reward function \label{algline: steering reward}
        $$R^{t}_\z(\cdot)\gets R_{\pi_*^{k}}(\cdot) + \|R_{\pi_*^{k}}(\cdot)\|_\infty\bm1.$$
        \State Agents play the $t$-th game.\label{algline: agents play with reward} %
        \State Obtain trajectory $((s^{t}_h,a^{t}_h))_{h=1}^H$. %
        \If {$\exists(h,s,a),~s.t.~n_k(h,s,a)\geq N_k(h,s,a)$}
            \State Update $\cP^{k+1}$ as in (\ref{eq: update transition confidence set}).
            \State $T_k\gets t$; $k\gets k+1$.
            \State
            $
            \pi_*^{k},\hat{M}^{k}\gets \argmax_{\pi\in\Pi,\hat M:\hat\PP_{ \hat M}\in\cP^{k}} U(\sad^\pi_{\hat{M}}).
            $\label{line:opt_policy}
        \EndIf
    \EndFor
\end{algorithmic}
\end{algorithm}

\textbf{Low Policy Switching Optimistic Exploration Strategy}\quad
For efficient exploration, we maintain a confidence set for $\PP^*$ denoted by $\cP$:
\begin{align}\label{eq: update transition confidence set}
    &\bar\PP^{k+1}_h(s'|s,a):=\sum_{t=1}^{T_k}\frac{\II\{s^{t}_h=s,a^{t}_h=a,s^{t}_{h+1}=s'\}}{\max\{1,N_{k+1}(h,s,a)\}},\nonumber\\
    &\cP^{k+1}:=\bigg\{\hat\PP:\forall h,s,a.\Vert\hat\PP_h(\cdot|s,a)-\bar\PP^{k+1}_h(\cdot|s,a)\Vert_1\nonumber\\
    &\qquad\qquad\qquad\qquad\qquad\qquad\leq\epsilon_{k+1}(h,s,a)\bigg\},
\end{align}
where $\epsilon_{k+1}(h,s,a):=\sqrt{\frac{2S\ln(THSA/\delta)}{\max\{1,N_{k+1}(h,s,a)\}}}$.
We highlight that we only update $\cP$ and switch target policy in low frequency, and here we use index $1\leq k\leq K$ to count the policy switching episodes, to distinguish with the steering steps $t\in[T]$.
We use $k(t)$ to denote the index of episode at iteration $t$ and use $T_k$ to denote the iteration number at the end $k$-th policy switching. We define $n_k(h,s,a)$ to be the number of samples equal to $(s,a)$ at time $h$ in episode $k$, and $N_k(h,s,a)=\sum_{k'<k}n_{k'}(h,s,a)$. A new episode begins as soon as we have as many samples in this episode as in all the previous ones for some $h,s,a$, i.e., $n_k(h,s,a)\geq N_k(h,s,a)$.
The main motivation for this technique is to avoid the agents' potentially adversarial behaviors.
As we will see later in the proof sketch, $K$ will appear in the steering gap upper bound.

For exploration, we select the optimistic policy $\pi_*^{(\cdot)}$ and model $\hat{M}^{(\cdot)}$ (line~\ref{line:opt_policy}) s.t. the induced density maximizes utility.
Then, we choose steering reward $R_\z$ to guide the agents towards $\pi_*^{(\cdot)}$ and collect data samples to update the model confidence set.
Intuitively, either $\pi_*^{(\cdot)}$ indeed maximizes the utility, implying a low steering gap; or the exploration helps to reduce the uncertainty.

\textbf{Managing the steering gap and cost}\quad
We have the following guarantees for Alg.~\ref{alg: incentive design}
\begin{restatable}{theorem}{ThmSteeringGapNoReward}\label{thm: incentive designer regret}
    Suppose the intrinsic reward $r^* = 0$, under Assump.~\ref{ass: adaptive regret} and \ref{assump:Lipschitz}, if we run Alg.~\ref{alg: incentive design} with $\delta\in(0,1)$, then with probability at least $1-2\delta$, $K \leq HSA\log_2 T$, and
    \begin{align*}
        &\Delta_T(\{\bmu^{t}_{M^*}\}_{t=1}^T)\leq L_U\sqrt{H^3SATK\adareg(T)}\\
        &\qquad\qquad\quad+36L_UH^3S\sqrt{\ln(THSA/\delta)AT} = o(T).\\
        &C_T(\{\bmu^{t}_{M^*},R_{\z}^{t}\}_{t=1}^T)\leq 4H\sqrt{TK\adareg(T)} = o(T).
    \end{align*}
\end{restatable}
As a concrete example, agents following Online Gradient Descent with step size $O(1/\sqrt{t})$ \citep{hazan_introduction_2023} result in $\adareg(T)=\tilde\cO(\sqrt{T})$ (ignoring $H,S$ and $A$), which implies $\tilde\cO(T^{3/4})$ steering gap.
Besides, if all the agents are capable enough s.t. for any $t\in[T]$, $\pi^{1,t},...,\pi^{N,t}$ are equilibria w.r.t. $r^* + R_{\z}^{t}$, $\adareg$ would be constant-level, resulting in a $\tilde{\cO}(\sqrt{T})$ bound.

\textbf{Proof Sketch}\quad
We first analyze the steering gap. 
Intuitively, Alg.~\ref{alg: incentive design} can be interpreted as a ``$K$-stage'' version of what we did in Sec.~\ref{sec: policy incentivization}.
In each stage, we pick a target policy, and steer the agents towards it for exploration.
Following this intuition, and thanks to the Lipschitz condition (Assump.~\ref{assump:Lipschitz}) and the optimism in planning, we can decompose the steering gap as follow:
\begin{align}
    \Delta_T(\{\bmu^{t}_{M^*}&\}_{t=1}^T)\leq L_U(2H+1)\underset{\Delta_{\text{est}}}{\underbrace{\sum_{t=1}^T\Vert\mu^{\bpi^{t}}_{\hat{M}^{k(t)}}-\bmu^{t}_{M^*}\Vert_1}}\nonumber\\
    +L_UH&\underset{\Delta_{\text{pop}}}{\underbrace{\sum_{t=1}^T \sum_{h,s}\bmu_{M^*,h}^{t}(s)\Vert\pi^{k(t)}_{*,h}(\cdot|s)-\bpi^{t}_h(\cdot|s)\Vert_1}}.\label{eq:decompose}
\end{align}
We refer the first term $\Delta_{\text{est}}$ as model estimation error, which measures the gap between the population density $\bmu^{t}$ and the density induced by the population average policy $\bpi^{t}$ (see definition in Sec.~\ref{chapter: preliminaries}) in the estimated model $\hat{M}^{k}$.
As we collect more and more data, $\hat{M}^{k}$ gets closer to $M^*$, and we can show $\Delta_{\text{est}}$ only grows sub-linearly.
The second term $\Delta_{\text{pop}}$ can be interpreted as the population convergence error, which is determined by how fast the agents converge to the target policy we steer them to.
Following the similar techniques in the proof of Thm.~\ref{thm: regret of policy reward}, $\Delta_{\text{pop}}$ can be upper bounded by:
\begin{align}
    \textstyle \vphantom{\underbrace{\sum_t^T}_{\texttt{AgentReg}}} %
    \sqrt{
    \vphantom{\sum_t^T}
    \smash[b]{HSAT\! \underbrace{\sum_{t=1}^T\langle R_{\pi_*^{k(t)}}(\bmu^{t}_{M^*}),\sad_{M^*}^{\pi_*^{k(t)}}-\bmu^{t}_{M^*}\rangle}_{\texttt{AgentReg}}\,}
    }.\label{eq:dyn_regret}
\end{align}
Here we use $\texttt{AgentReg}$ to refer the summation term, which can be interpreted as the agents' dynamic regret if choosing $\sad_{M^*}^{\pi_*^{k(t)}}$ as the comparators. 
Thanks to the low policy switching, \texttt{AgentReg} can be controlled by $O(K\adareg(T))$, and the only remaining step is to control $K$. Note that we only switch policy when the number of visitation of some state-action pair got doubled, therefore, $K$ only grows in $O(\log(T))$.

For the steering cost, we can calculate that
\begin{align*}
    C(\bmu^{t}_{M^*}, R^{t}_{\z})
    \leq 2H\|R_{\pi_*^{k(t)}}(\bmu^{t}_{M^*})\|_\infty,
\end{align*}
and for any $\pi,\sad$, $\|R_\pi(\sad)\|_\infty\leq2\|(W^\pi-I)\sad\|_2$ which, by Eq.~\eqref{eq:connection}, is equal to $2\sqrt{\langle R_\pi(\sad),\sad-\sad^\pi\rangle}$.
Using Jensen's inequality and Assump.~\ref{ass: adaptive regret}, we derive the final bound.

\section{STEERING WITH NON-ZERO INTRINSIC REWARD}\label{sec: incentive design with unknown reward}
Next, we turn to \textbf{Scenario 2} in Sec.~\ref{sec:objectives}, the complete setting where the agents' pre-existing reward function $r^* \in [0, r_{\max}]$ is both non-zero and unknown.
The non-zero intrinsic reward introduces non-trivial additional challenges.
Firstly, it changes the steering landscape and introduces some prior bias for our steering reward design.
Secondly, since it is unknown, we must account for its interference on the steering dynamics and undertake strategic exploration to estimate $r^*$.
In the following, we explain how we overcome these challenges by a pessimism-based reward estimation strategy.

\textbf{Confidence set for $r^*$}\quad
We recall our setup in Sec.~\ref{sec: intro to incentive design}: the mediator can observe the population density $\bmu^{t}$ and noisy reward $r^t=(r^*_h(s_h^t,a_h^t,\bmu^t_{M^*,h})+\xi_h)_{h\in[H]}$ perturbed by i.i.d.\ zero-mean $\sigma$-sub-Gaussian noise $\xi$.
We will use this information to estimate the original reward.
At each iteration $t$, we maintain a confidence set $\hat\cR^{t}$ for $r^*$, defined by:
\begin{align}
    \hat\cR^{t}:=&\left\{\hat r\in\cR:\|\hat r-\bar r^{t}\|_{2,E_t}\leq\sqrt{\beta_t}\right\},\nonumber\\
    \bar r^{t}:=&\argmin_{\hat r\in\cR}\sum_{i=1}^{t-1}\sum_{h=1}^H\left(\hat r_h(s^i_h,a^i_h,\bmu^i_{M^*,h})-r_h^{i}\right)^2,\label{eq: reward confidence set}
\end{align}
where $\|g\|_{2,E_t}^2:=\sum_{i=1}^{t-1}\sum_{h=1}^H(g_h(s^i_h,a^i_h,\bmu^i_{M^*,h}))^2$ for any function $g$ as a short note. We use $\beta_t$ to denote confidence interval length to ensure $r^*$ is contained in the confidence set at any time with high probability. 
We defer a detailed choice of $\beta_t$ to Lem.~\ref{lem: true unknown reward in confidence set}. Informally, $\beta_t = O(\sigma^2 \log N(\cR,\frac{1}{T}))$ grows in $\log T$, where $N(\cR,\epsilon)$ is the $\epsilon$-covering number of $\cR$.

\textbf{Steering Reward Design with Pessimism}\quad
We consider the following steering reward design
\begin{align}\label{eq: reward modification}
    \forall \mu\in\Delta_{\cS\times\cA}^H,~&R^{t}_{\nz}(\mu):=R_{\pi_*^{k(t)}}(\mu)-(\bar r^{t}(\mu)-w_{\hat\cR^{t}}(\mu))\nonumber\\
    &+(r_{\max}+\|R_{\pi_*^{k(t)}}(\mu)\|_{\infty})\bm1
\end{align}
Here $\pi_*^{k(t)}$ is computed in the same way as Alg,~\ref{alg: incentive design}; $\bar{r}^t\in\hat\cR^{t}$ (defined in Eq.~\eqref{eq: reward confidence set}) is the reward estimation achieving the minimal empirical loss; $w_{\hat\cR^{t}}(\sad)$ is a vector with elements $(w_{\hat\cR^{t}}(\sad))_{h,s,a}:=\sup_{r,\tilde r\in\hat\cR^{t}}\left|r_h(s,a,\sad)-\tilde r_h(s,a,\sad)\right|$, which quantifies the estimation uncertainty for each state-action pair; the last constant shift term ensures non-negativity.

As we can see, the main difference compared with steering reward $R^{t}_\z$ in Alg.~\ref{alg: incentive design} is that we include an additional reward estimation term to offset the effect by the non-zero original reward $r^*$. In this way, the agents will follow the guidance by $R_{\pi_*^{k(t)}}$ to explore as we want.
Note that here we conduct a \emph{pessimism-based} reward estimation such that $\bar r^{t}-w_{\hat\cR^{t}} \leq r^*$ for some technical reason, which we will explain later.

\textbf{Steering Algorithm Design}\quad
The algorithm design for the non-zero intrinsic reward setting only differs from Alg.~\ref{alg: incentive design} in the additional update of $\hat\cR^{t}$ as in Eq.~\eqref{eq: reward confidence set} and choosing Eq.~\eqref{eq: reward modification} as the steering reward $R_{\nz}^{t}$.
For completeness, we defer the detailed algorithm to Alg.~\ref{alg: incentive design unknown reward} in Appx.~\ref{appx:alg_nonzero_reward}.
We have the following guarantees for steering gap and steering cost.

\begin{restatable}{theorem}{ThmUtilityRegretUnknownReward}\label{thm: incentive designer regret unknown reward}
    Under Assump.~\ref{assump:realizability},~\ref{ass: adaptive regret} and~\ref{assump:Lipschitz}, if we run Alg.~\ref{alg: incentive design unknown reward} with $0<\delta<1$, then with probability at least $1-6\delta$, $K\leq HSA\log_2T$, and
    \begin{align*}
        \Delta_T(\{\bmu^{t}_{M^*}\}_{t=1}^T)
        \leq\,& L_U\sqrt{H^3SAT(K\adareg(T)+D)}\\
        &+36L_UH^3S\sqrt{AT\ln(THSA/\delta)},\\
        C_T(\{\bmu^{t}_{M^*},R_{\nz}^{t} - (&r_{\max}\cdot\bm1 - r^*)\}_{t=1}^T)\\
        =4H&\sqrt{T(K\adareg(T)+D)}+D,
    \end{align*}
    where $D = \tilde{O}(\sqrt{\beta_TH\dim_E(\cR,T^{-1})T}))$.
\end{restatable}
Comparing with Theorem~\ref{thm: incentive designer regret}, we can find both the steering gap and cost only differ in the additional term $D$, which results from the estimation error of $r^*$. The term $D$ depends on the Eluder dimension of $\cR$ and $\beta_T$. In Appx.~\ref{appx:example_function_class}, we show several common function classes with $\dim_E(\cR,T^{-1})\in\tilde\cO(1)$, and where by choosing $\beta_T$ appropriately, we have $D\in\tilde\cO(\sqrt{T})$.
As a result, both the steering gap and cost upper bounds in Thm.~\ref{thm: incentive designer regret unknown reward} will be sub-linear in $T$.

\textbf{Proof Sketch}\quad
Similar to the proof for Thm.~\ref{thm: incentive designer regret}, we can decompose the steering gap as Eq.~\eqref{eq:decompose}, and upper bound model estimation error term $\Delta_{\text{est}}$ in the same way. The proof diverges when we upper bound \texttt{AgentReg} in Eq.~\eqref{eq:dyn_regret}, because the agents' no-regret behavior holds for $r^* + R_{\nz}^{t}$ in this setting. We can write 
\begin{align*}
    \texttt{AgentReg}
    &=\sum_{t=1}^T \langle R_{\nz}^{t}(\bmu^t_{M^*}) + r^*(\bmu^t_{M^*}) - r^*(\bmu^t_{M^*})\\
    &\quad+\bar r^{t}(\bmu^t_{M^*}) - w_{\hat\cR^{t}}(\bmu^t_{M^*}),\sad_{M^*}^{\pi_*^{k(t)}}-\bmu^{t}_{M^*}\rangle.
\end{align*}
Using pessimism, i.e., $r^*\geq\bar r^{t} - w_{\hat\cR^{t}}$, we can bound this by
\begin{align*}
    &\sum_{t=1}^T \langle R_{\nz}^{t}(\bmu^t_{M^*}) + r^*(\bmu^t_{M^*}),\sad_{M^*}^{\pi_*^{k(t)}}-\bmu^{t}_{M^*}\rangle\\
    &\quad+\sum_{t=1}^T\langle r^*(\bmu^{t}_{M^*}) - \bar r^{t}(\bmu^t_{M^*}) + w_{\hat\cR^{t}}(\bmu^t_{M^*}),\bmu^{t}_{M^*}\rangle.
\end{align*}

Clearly, the first term above is just agents' dynamic regret regarding the total reward they received and can be bounded again by $K\adareg(T)$. 
The second term above can be further controlled by $\cO(\sum_t \langle w_{\hat\cR^{t}}(\bmu^t),\bmu^{t} \rangle)$, which is basically the accumulative confidence interval length for reward estimation and its growth can be controlled by Eluder dimension (Lem.~\ref{lem: eluder confidence set bound}) and is only sub-linear in $T$.

For the steering cost, we can provide an upper bound involving \texttt{AgentReg} and reward estimation error that we analyzed before. To save space, we do not repeat it here and refer the reader to Appx.~\ref{appx:proof_non_zero_rew} for the full proof.

\begin{remark}
    Our strategy to deal with the intrinsic reward $r^*$ is to try to ``cancel'' it with our steering reward. This approach is justified by the fact that we keep $r^*$ and $U$ very general, which means that the target density to maximize $U$ may not coincide with an equilibrium associated with the original reward $r^*$. Therefore, to ensure the target density is still a stationary point for no-regret learners, we treat $r^*$ as a competing force to offset.
    We admit that there might be other options to counteract the impact of $r^*$ with lower steering costs, and we leave further investigation to the future work.
\end{remark}
\begin{remark}[Generalization to Unknown Utility Setting]\label{rem:unknown utility}
    Although this paper focuses on the case when $U$ is revealed to the mediator, it is possible to generalize our results to the case where the utility function $U$ is unknown, but it lies in a known function class $\cU$ with bounded Eluder dimension.
    In Appx.~\ref{appx:unknown utility}, we formalize this setting and present a solution to address this case based on a simple modification of the current methods. 
    Our established regret bound for steering gap and steering cost grow at a rate of $\tilde\cO(T^{5/6})$.
    Although the results are worse than the rate of $\tilde\cO(T^{3/4})$ in Thm.~\ref{thm: incentive designer regret unknown reward} due to the challenges in exploring the utility function, they are still sub-linear in $T$.
\end{remark}

%% file: discussion.tex
\section{CONCLUSION}

We study a novel problem setting for incentive design in unknown mean-field games with no-regret agents. Our optimistic algorithm introduces newly developed steering reward designs, achieving sublinear utility regret and steering costs when the intrinsic reward is zero. Extending to the setting with a non-zero and unknown intrinsic reward function, we adapted our algorithm to handle this new challenge, maintaining sublinear utility regret and vanishing steering costs competing with a baseline strategy.
Future work could explore the more challenging case where the transition function is also dependent on the population density. Another interesting direction is to identify better or even optimal steering reward design to stabilize the target policy and design an algorithm with sub-linear guarantees comparing with that benchmark.

%% file: Appendix/supplement.tex
\section{TABLE OF FREQUENTLY USED NOTATIONS}
\begin{center}
\begin{tabular}{ l l }
\hline
\textbf{Notation} & \textbf{Description}\\
\hline
$[n]$ & $\{1,2,...,n\}$ for any $n\in\NN$\\
$\Delta_\cX$ & Set of probability distributions over a finite set $\cX$\\
$\II\{\cE\}$ & Indicator function for the event $\cE$\\
$\bm1$ & All-one vector\\
$\be_i$ & The $i$-th standard-basis vector\\
$M=(N,\cS,\cA,H,\PP_M,r_M,\mu_1)$ & The model / game\\
$N$ & Number of agents\\
$\cS,\cA$ & State and action space\\
$H$ & Horizon length of the game\\
$\mu_1$ & Initial state distribution\\
$\{\PP_{M,h}:\cS\times\cA\to\Delta_\cS\}_{h\in[H]}$ & Transition function\\
$\{r_{M,h}:\cS\times\cA\times\Delta_{\cS\times\cA}\to[0,r_{\max}]\}_{h\in[H]}$ & Reward function\\
$\{R_{h}:\cS\times\cA\times\Delta_{\cS\times\cA}\to\RR\}_{h\in[H]}$ & Steering reward function (capitalized)\\
$r:\Delta_{\cS\times\cA}^H\to\RR^{HSA}$ & Vectorized reward function $(r(\mu))_{h,s,a}=r_h(s,a,\mu_h)$ \\
$\{\pi_h:\cS\to\Delta_\cA\}_{h\in[H]}$ & Markov policy\\
$\Pi$ & Set of all policies\\
$\sad_M^{\pi}$ & State-action density of policy $\pi$ in model $M$\\
$\Psi_M$ & Set of possible state-action densities in model $M$\\
$\adareg(T)$ & Adaptive regret bound after $T$ games\\
$U:\Delta_{\cS\times\cA}^H\to\RR$ & Utility function\\
$C(\bmu^t,R^t)=\langle R^t(\bmu^t),\bmu^t\rangle$ & Steering cost function\\
$R_\pi$ & Reward function which incentivizes policy $\pi$ \\
$M^*,r^*,\PP^*$ & True model, intrinsic reward, transition function \\
$R_{\z}$ & Steering reward for the setting where $r^*=0$. \\
& ``z'' in sub-scription as a short note of ``zero''. \\
$R_{\nz}$ & Steering reward for the setting where $r^*\in\cR$. \\
& ``nz'' in sub-scription as a short note of ``zero'' \\
$\dim_E(\cF,\epsilon)$ & Eluder dimension of function class $\cF$\\
$\bmu$ & Population density $\bmu := \frac{1}{N}\sum_n \mu^{\pi^n}$\\
$\bpi$ & Population average policy induced by $\bmu$\\
$\cO,\tilde{\cO}$ & Standard big-O notations \\
\hline
\end{tabular}
\end{center}

\input{Appendix/MainResultsSummary}

\input{Appendix/Other_Related_Works}

\section{REGARDING NO-ADAPTIVE REGRET ASSUMPTION}

\subsection{Proof Of Proposition~\ref{prop:no_regret_population}}
\PropNoRegretPopulation*
\begin{proof}
    We have
    \begin{align*}
        \sup_{1\leq a<b\leq T}\max_{\mu\in\Psi_{M^*}}\sum_{t=a}^b\langle r^*(\bmu^t)+R^t(\bmu^t),\mu-\bmu^t\rangle
        =\sup_{1\leq a<b\leq T}\max_{\mu\in\Psi_{M^*}}\sum_{t=a}^b\langle r^*(\bmu^t)+R^t(\bmu^t),\mu-\frac1N\sum_{n=1}^N\mu^{\pi^{n,t}}\rangle\\
        \leq\frac1N\sum_{n=1}^N\sup_{1\leq a<b\leq T}\max_{\mu\in\Psi_{M^*}}\sum_{t=a}^b\langle r^*(\bmu^t)+R^t(\bmu^t),\mu-\mu^{\pi^{n,t}}\rangle
        \leq\frac1N\sum_{n=1}^N\adareg(T)
        =\adareg(T),
    \end{align*}
    where we used Assumption~\ref{ass: adaptive regret} in the third step.
\end{proof}

\subsection{Concrete Examples Satisfying No-Adaptive Regret Assumption}\label{appx:concrete_example}
In this section, we provide some concrete agents learning dynamics examples to support our arguments on the practicality of Assump.~\ref{ass: adaptive regret}.

\paragraph{Example 1: Colluded Agents with Full Observation of $R^{t}$}
If the agents are able to observe the mediator's steering strategy $R^{t}$ and $R^{t}$ is Lipschitz in density (which is indeed satisfied by our proposed algorithms), the agents can collude together and take a (approximate) Nash Equilibrium policy induced by the reward function $r^* + R^{t}$, which is guaranteed to be exist given the Lipschitz condition \citep{huang_statistical_2023}.
By the definition of Nash, each agent will have non-positive adaptive regret, which satisfies Assump.~\ref{ass: adaptive regret}.

Note that in the contract design literature, it is usually assumed the agents are able to do best response \citep{ho2014adaptive,zhu2022sample} to the principal's (mediator's) strategy if there is only one agent, or take the equilibrium policies for many agents setting \citep{carmona2021finite, elie2019tale}. 
Based on the discussion above, those assumptions are strictly stronger than and implies our no-adaptive-regret assumption.

\paragraph{Example 2: Independent Agents Conducting Online Convex Learning}
In this second example, we consider less powerful agents who can not observe the entire $R^{t}$ or coordinate with the other agents.
Note that from an agent's perspective, the interaction protocol in Procedure~\ref{alg: agent-mediator interaction} can be interpreted as an online linear optimization task, as in Procedure~\ref{alg: agent-adversary interaction}.

\begin{algorithm}
\makeatletter
\renewcommand{\ALG@name}{Procedure}
\makeatother
\caption{Agent-adversary interaction}
\label{alg: agent-adversary interaction}
\begin{algorithmic}[1]
    \For{$t=1,...,T$}
        \State Agent chooses $x_t\in\cX$, where $\cX\subseteq\RR^d$ is a convex set in Euclidean space.
        \State Adversary chooses a reward vector $r_t\in\RR^d$, possibly based on the history and $x_t$.
        \State Agent observes $r_t$ and obtains reward $\langle r_t,x_t\rangle$.
    \EndFor
\end{algorithmic}
\end{algorithm}

In our setting, in each iteration $t\in[T]$, the agents pick a density (by picking a policy) from the convex set $\Psi_{M^*}$ and receive potentially adversarial feedback $R^{t}(\bmu^{t})$ (or $\langle R^{t}(\bmu^{t}),\sad^{\pi^{n,t}}\rangle$ in bandit feedback setting).
Then, Assump.~\ref{ass: adaptive regret} coincides with the standard no-adaptive regret guarantees in online convex optimization setting.
Therefore, Assump.~\ref{ass: adaptive regret} can be realized if each agent independently adopts any no-adaptive regret online learning algorithm \citep{AdaptiveRegret,hazan_introduction_2023}.

As a concrete algorithm choice, online gradient descent (OGD) achieves a external regret bound of $\frac32GD\sqrt{T}$ \citep{hazan_introduction_2023}, where $D\leq2H$ is the diameter of $\cX=\Psi_{M^*}$ and $G$ an upper bound on $\Vert r_t\Vert_2\leq\sqrt{d}\Vert r_t\Vert_\infty$. In our case, we can bound $G\leq\sqrt{HSA}(r_{\max}+R_{\max})$. A bound for $r_{\max}+R_{\max}$ is discussed in Appendix~\ref{appx : boundedness of steering rewards}.
Moreover, in the full feedback setting (the agents know the model $M$ and are able to observe $R^{t}(\bmu^{t})$), the no-adaptive-regret assumption is not much stronger than no-external-regret, as is demonstrated by the following proposition.
\begin{proposition}[Theorem 1.3 of \citet{AdaptiveRegret}]\label{prop:convertion_external_adaptive}
    Let $(r_t)_{t=1}^T$ be reward vectors in $[0,C]^d$. Any algorithm following Protocol~\ref{alg: agent-adversary interaction} with external regret $\reg(T)$ can be utilized to build an algorithm with adaptive regret at most $\reg(T)+\cO(C\sqrt{T\log T})$.
\end{proposition}
Thus, Assump.~\ref{ass: adaptive regret} can be satisfied with an adaptive regret bound of $\tilde\cO(\sqrt{T})$ if all the agents follow OGD, modified as in Prop.~\ref{prop:convertion_external_adaptive}.

\subsection{Motivating Adaptive Regret}\label{sec: steering in multiple directions}
Here, we show a small example that should motivate why we need the no-adaptive-regret assumption instead of no-external-regret. External regret is one of the most common regret types, and it is the same as adaptive regret in Assumption~\ref{ass: adaptive regret}, but $a=1,b=T$ are fixed. If we want to steer the agents in different directions, the no-external-regret assumption might not be enough, as we can see in the following example.

Consider the stateless setting with $|\cA|=2$, where the incentive designer deploys $R(\sad)=\be_1$ for the first $T/2$ iterations and $R(\sad)=\be_2$ for the remaining $T/2$ iterations.

Suppose all the agents perform the Hedge algorithm, where
\begin{align*}
    \bmu^{t}(a)=\bpi^t(a)=\frac1Z_t\exp\left(1+\eta\sum_{s=1}^{t-1}\langle R^{s}(\bmu^{s}),\be_a\rangle\right),
\end{align*}
and $Z_t$ is the normalizing constant. This algorithm is known to have sublinear external regret \citep{HedgeAlgorithm}. The population density at iteration $t\geq T/2$ is
\begin{align*}
    \bmu^{t}=\frac1Z_t\left(\begin{matrix}\exp(1+\eta T/2)\\\exp(1+\eta(t-T/2))\end{matrix}\right),
\end{align*}
while the optimal action is $\be_2$. Thus, over the interval $[T/2+1,T]$, the agents accumulate expected regret
\begin{align*}
    \sum_{t=T/2+1}^T(1-\bmu^{t}(2))=\sum_{t=T/2+1}^T\underset{\geq1/2}{\underbrace{\frac{\exp(1+\eta T/2)}{\exp(1+\eta T/2)+\exp(1+\eta(t-T/2)}}}\geq T/4.
\end{align*}
So, although this algorithm has no external regret, we still might have to wait $\Omega(T)$ many rounds to let the agents converge to a different density. One can easily observe that with the no-adaptive-regret assumption, this is not an issue.

\subsection{Boundedness Of Steering Rewards}\label{appx : boundedness of steering rewards}
As we see in Assumption~\ref{ass: adaptive regret}, the adaptive regret bound $\adareg(T)$ is dependent on $r_{\max}+R_{\max}$. In this section, we show that $R_{\max}=\cO(1+r_{\max})$ for both of our steering rewards $R_{\z}$ and $R_{\nz}$.

\begin{proposition}\label{prop: bounded policy steering reward}
    For any $\pi\in\Pi$ and $\mu\in\Psi$,
    $
    \|R_\pi(\mu)\|_\infty\leq2
    $,
    where $R_\pi$ is defined as in Eq.~\eqref{eq: policy incentive reward}.
\end{proposition}
\begin{proof}
    As one can observe using the definition of $R_\pi$ in Eq.~\eqref{eq: policy incentive reward} and $W^\pi$ in Eq.~\eqref{eq: policy reward matrix}, we have for any $h,s,a$,
    \begin{align*}
        &\left|(R_\pi(\mu))_{h,s,a}\right|
        =\left|(\mu^\top(W^\pi-I)^\top(W^\pi-I))_{h,s,a}\right|
        =\left|(W^\pi\mu-\mu)^\top(W^\pi-I)_{(h,s,a)}\right|\\
        &=\left|\sum_{a'}(\pi_h(a'|s)\mu_h(s)-\mu_h(s,a'))(\pi_h(a'|s)-\II\{a'=a\})\right|
        \leq\sum_{a'}\underbrace{|\pi_h(a'|s)\mu_h(s)-\mu_h(s,a')|}_{\leq 1}\cdot|\pi_h(a'|s)-\II\{a'=a\}|\\
        &\leq\sum_{a'\neq a}\pi_h(a'|s)+|\pi_h(a|s)-1|\leq2,
    \end{align*}
    where $(W^\pi-I)_{(h,s,a)}$ is the $(h,s,a)$-th column of $W^\pi-I$.
\end{proof}

\begin{proposition}
    We have for $R_{\z}$, as defined in Alg.~\ref{alg: incentive design}, and for $R_{\nz}$, as defined in Eq.~\eqref{eq: reward modification}, that for all iterations $t\in[T]$,
    \begin{align*}
        \|R_{\z}^t(\mu)\|_\infty\leq4
        \quad\text{and}\quad
        \|R_{\nz}^t(\mu)\|_\infty\leq2r_{\max}+4
        \quad\forall\mu\in\Delta_{\cS\times\cA}^H.
    \end{align*}
\end{proposition}
\begin{proof}
    We have $\|R_{\z}^t(\mu)\|_\infty=\|R_{\pi_*^{k(t)}}(\mu)+\|R_{\pi_*^{k(t)}}(\mu)\|_\infty\bm1\|_\infty\leq2\|R_{\pi_*^{k(t)}}(\mu)\|_\infty$, which is at most $4$, by Prop.~\ref{prop: bounded policy steering reward}.

    By the definition of $w_{\hat\cR^t}$, we know that the elements of $w_{\hat\cR^t}(\mu)$ are bounded in $[0,r_{\max}]$ for any $\mu$. Therefore, and by Prop.~\ref{prop: bounded policy steering reward},
    \begin{align*}
        \left\|R_{\nz}^t(\mu)\right\|_\infty
        &=\left\|R_{\pi_*^{k(t)}}(\mu)-(\bar r^{t}(\mu)-w_{\hat\cR^{t}}(\mu))+(r_{\max}+\|R_{\pi_*^{k(t)}}(\mu)\|_{\infty})\bm1\right\|_\infty\\
        &\leq
        \left\|R_{\pi_*^{k(t)}}(\mu)+\|R_{\pi_*^{k(t)}}(\mu)\|_{\infty}\bm1\right\|_\infty
        +\left\|\bar r^{t}(\mu)-w_{\hat\cR^{t}}(\mu)\right\|_\infty
        +\left\|r_{\max}\bm1\right\|_\infty\\
        &\leq4+2r_{\max}.
    \end{align*}
\end{proof}

\newpage
\section{STATE-ACTION DENSITY}

\subsection{$\Psi_M$ Is Convex}
\begin{lemma}\label{lem: state-action distribution set}
    \begin{align*}
        \Psi_M=\{\sad:\sad\geq0,\sum_{a'}\sad_{h+1}(s,a')=\sum_{s',a'}\PP_{M,h}(s|s',a')\sad_h(s',a')\forall h,s,\sum_{a'}\sad_1(s,a')=\sd_1(s)\}
    \end{align*}
\end{lemma}
\begin{proof}
    We abbreviate $\PP=\PP_M$ and $\sad=\sad_M$, since the model is fixed throughout. For $\sad\in\Psi_M$, it is easy to see that the conditions on the right-hand side are fulfilled. The other direction is more involved. Suppose $\tilde\sad$ fulfills $\tilde\sad\geq0$ and for all $s,h$, $\sum_{a'}\tilde\sad_{h+1}(s,a')=\sum_{s',a'}\PP_h(s|s',a')\tilde\sad_h(s',a')$ as well as $\sum_{a'}\tilde\sad_1(s,a')=\sd_1(s)$. Now, define $\pi$ such that for all $s,a,h$,
    \begin{align*}
        \pi_h(a|s)=\begin{cases}\frac{\tilde\sad_h(s,a)}{\sum_{a'}\tilde\sad_h(s,a')},&\text{if }\sum_{a'}\tilde\sad_h(s,a')\neq0\\1/A,&\text{else}\end{cases}.
    \end{align*}
    Clearly, $\pi_h(a|s)\geq0$ and $\sum_a\pi_h(a|s)=1$, which means $\pi\in\Pi$. First of all,
    \begin{align*}
        \sad_1^{\pi}(s,a)=\pi_1(a|s)\sd_1(s)=\frac{\tilde\sad_1(s,a)}{\sum_{a'}\tilde\sad_1(s,a')}\sd_1(s)=\tilde\sad_1(s,a).
    \end{align*}
    By induction, for all $h\geq1$ we have if $\sum_{a'}\tilde\sad_{h+1}(s,a')\neq0$,
    \begin{align*}
        \sad_{h+1}^{\pi}(s,a)
        &=\sum_{s',a'}\pi_{h+1}(a|s)\PP_h(s|s',a')\sad_h^{\pi}(s',a')\\
        &=\pi_{h+1}(a|s)\sum_{s',a'}\PP_h(s|s',a')\tilde\sad_h(s',a')\\
        &=\frac{\tilde\sad_{h+1}(s,a)}{\sum_{a'}\tilde\sad_{h+1}(s,a')}\sum_{a'}\tilde\sad_{h+1}(s,a')=\tilde\sad_{h+1}(s,a),
    \end{align*}
    and in case $\sum_{a'}\tilde\sad_{h+1}(s,a')=0$, we know that $\tilde\sad_{h+1}(s,a)=0$ and therefore
    \begin{align*}
        \sad_{h+1}^{\pi}(s,a)
        &=\sum_{s',a'}\pi_{h+1}(a|s)\PP_h(s|s',a')\sad_h^{\pi}(s',a')\\
        &=\pi_{h+1}(a|s)\sum_{s',a'}\PP_h(s|s',a')\tilde\sad_h(s',a')\\
        &=1/A\sum_{a'}\tilde\sad_{h+1}(s,a')=0=\tilde\sad_{h+1}(s,a).
    \end{align*}
    We can conclude that $\tilde\sad=\sad^{\pi}$ and thus $\tilde\sad\in\Psi_M$.
\end{proof}

\begin{lemma}\label{lem: state-action distribution constraints}
    \begin{align*}
        \Psi_M=\{\sad:\sad\geq0,B\sad=b\},
    \end{align*}
    where
    \begin{align*}
        B=\left(\begin{matrix}D&&&&\\-\PP_{M,1}^\top&D&&&\\&-\PP_{M,2}^\top&D&&\\&&...&&\\&&&-\PP_{M,H-1}^\top&D\end{matrix}\right),\quad b=\left(\begin{matrix}\sd_1\\0\\\vdots\\0\end{matrix}\right),
    \end{align*}
    $D:=I_S\otimes\bm1_A^\top$ ($\otimes$ is the tensor product) and $\PP_M$ is viewed as a matrix such that $(\PP_{M,h})_{(s,a),s'}=\PP_{M,h}(s'|s,a)$. An immediate consequence of this formulation is that $\Psi_M$ is convex.
\end{lemma}
\begin{proof}
    This result is simply a reformulation of Lemma~\ref{lem: state-action distribution set}. We can rewrite the condition $\sum_{a'}\sad_{h+1}(s,a')=\sum_{s',a'}\PP_{M,h}(s|s',a')\sad_h(s',a')$ as $D\sad_{h+1}=\PP_{M,h}^\top\sad_h$. The condition $\sum_{a'}\sad_1(s,a')=\sd_1(s)$ can be written as $D\sad_1(\cdot,\cdot)=\sd_1$.
\end{proof}

\subsection{Inequalities}

\begin{restatable}{lemma}{LemLipschitzOfSad}\label{lem: Lipschitzness of state-action density function}
    For any model $M$ and any $\pi,\tilde\pi\in\Pi$,
    \begin{align*}
        \|\sad_M^{\pi}-\sad_M^{\tilde\pi}\|_1\leq H\sum_{h,s}\sd_{M,h}^\pi(s)\|\pi_h(\cdot|s)-\tilde\pi_h(\cdot|s)\|_1,
    \end{align*}
    where $\sd_{M,h}^\pi(s)=\sum_a\sad_{M,h}^{\pi}(s,a)$.
\end{restatable}

\begin{proof}
    Since the model $M$ is fixed throughout, we abbreviate $\sad=\sad_M$ and $\PP=\PP_M$. First of all, $\|\sad_1^{\pi}-\sad_1^{\tilde\pi}\|_1=\sum_{s,a}\sd_1(s)|\pi_1(a|s)-\tilde\pi_1(a|s)|$. Furthermore, for any $h$,
\begin{align*}
    \|\sad_{h+1}^{\pi}-\sad_{h+1}^{\tilde\pi}\|_1
    &=\sum_{s,a}|\sad_{h+1}^{\pi}(s,a)-\sad_{h+1}^{\tilde\pi}(s,a)|\\
    &=\sum_{s,a}\left|\pi_{h+1}(a|s)\sum_{s',a'}\sad_h^{\pi}(s',a')\PP_h(s|s',a')-\tilde\pi_{h+1}(a|s)\sum_{s',a'}\sad_h^{\tilde\pi}(s',a')\PP_h(s|s',a')\right|\\
    &\leq\sum_{s,a}\left|\pi_{h+1}(a|s)-\tilde\pi_{h+1}(a|s)\right|\sum_{s',a'}\sad_h^{\pi}(s',a')\PP_h(s|s',a')\\
    &~+\sum_{s,a}\tilde\pi_{h+1}(a|s)\sum_{s',a'}\left|\sad_h^{\pi}(s',a')-\sad_h^{\tilde\pi}(s',a')\right|\PP_h(s|s',a')\\
    &=\sum_{s,a}\sd_{h+1}^\pi(s)\left|\pi_{h+1}(a|s)-\tilde\pi_{h+1}(a|s)\right|
    +\sum_{s',a'}\left|\sad_h^{\pi}(s',a')-\sad_h^{\tilde\pi}(s',a')\right|\sum_{s,a}\tilde\pi_{h+1}(a|s)\PP_h(s|s',a')\\
    &=\sum_s\sd_{h+1}^\pi(s)\|\pi_{h+1}(\cdot|s)-\tilde\pi_{h+1}(\cdot|s)\|_1+\|\sad_h^{\pi}-\sad_h^{\tilde\pi}\|_1.
\end{align*}
By induction,
\begin{align*}
    \|\sad_h^{\pi}-\sad_h^{\tilde\pi}\|_1\leq\sum_{h'=1}^h\sum_s\sd_{h'}^\pi(s)\|\pi_{h'}(\cdot|s)-\tilde\pi_{h'}(\cdot|s)\|_1.
\end{align*}
Finally,
\begin{align*}
    \|\sad^{\pi}-\sad^{\tilde\pi}\|_1
    &\leq\sum_{h=1}^H\sum_{h'=1}^h\sum_s\sd_{h'}^\pi(s)\|\pi_{h'}(\cdot|s)-\tilde\pi_{h'}(\cdot|s)\|_1\\
    &\leq H\sum_{h=1}^H\sum_s\sd_h^\pi(s)\|\pi_h(\cdot|s)-\tilde\pi_h(\cdot|s)\|_1.
\end{align*}
\end{proof}

\LemDensityDiffUBSpecial*
\begin{proof}
    Note that $\Psi_M$ is a convex set and $\bmu_M \in \Psi_M$. By definition, we have $\bmu_M = \mu_M^{\bpi}$. By applying Lem.~\ref{lem: Lipschitzness of state-action density function} for model policy $\bpi$ and $\pi$ in $M$, we finish the proof.
\end{proof}

\begin{lemma}\label{lem: state-action density difference}
    Consider any $\pi\in\Pi$ and models $M,\tilde M$, who are the same except with different transition functions $\PP,\tilde\PP$ respectively. Then,
    \begin{align*}
        \|\sad_{\tilde{M}}^{\pi}-\sad_M^{\pi}\|_1\leq H\sum_{h=1}^{H-1}\sum_{s,a}\sad_{M,h}^{\pi}(s,a)\|\tilde\PP_h(\cdot|s,a)-\PP_h(\cdot|s,a)\|_1.
    \end{align*}
\end{lemma}

\begin{proof}
Recall the definition of the state-action density function:
\begin{align*}
    \sad_{M,h+1}^{\pi}(s,a)=\sum_{s',a'}\pi_{h+1}(a|s)\PP_{M,h}(s|s',a')\sad_{M,h}^{\pi}(s',a')\quad\text{and}\quad\sad_1^{\pi}(s,a)=\pi_1(a|s)\sd_1(s).
\end{align*}

We abbreviate $\tilde\sad=\sad_{\tilde{M}}^{\pi},\sad=\sad_M^{\pi}$. Since $\sd_1$ is the same for $M$ and $\tilde{M}$, $\sum_{s,a}|\tilde\sad_1(s,a)-\sad_1(s,a)|=0$. Furthermore, for all $h$,
\begin{align*}
    &\|\tilde\sad_{h+1}-\sad_{h+1}\|_1=\sum_{s,a}\left|\tilde\sad_{h+1}(s,a)-\sad_{h+1}(s,a)\right|\\
    &=\sum_{s,a}\sum_{s',a'}\pi_{h+1}(a|s)\left|\tilde\PP_h(s|s',a')\tilde\sad_h(s',a')-\PP_h(s|s',a')\sad_h(s',a')\right|\\
    &=\sum_{s,a,s'}\left|\tilde\PP_h(s'|s,a)\tilde\sad_h(s,a)-\PP_h(s'|s,a)\sad_h(s,a)\right|\\
    &\leq\sum_{s,a,s'}\tilde\PP_h(s'|s,a)\left|\tilde\sad_h(s,a)-\sad_h(s,a)\right|+\sad_h(s,a)\left|\tilde\PP_h(s'|s,a)-\PP_h(s'|s,a)\right|\\
    &=\sum_{s,a}\left|\tilde\sad_h(s,a)-\sad_h(s,a)\right|+\sum_{s,a}\sad_h(s,a)\sum_{s'}\left|\tilde\PP_h(s'|s,a)-\PP_h(s'|s,a)\right|\\
    &=\|\tilde\sad_h-\sad_h\|_1+\sum_{s,a}\sad_h(s,a)\|\tilde\PP_{h'}(\cdot|s,a)-\PP_{h'}(\cdot|s,a)\|_1.
\end{align*}
Using induction on $h$, we obtain $\|\tilde\sad_h-\sad_h\|_1\leq\sum_{h'=1}^{h-1}\sum_{s,a}\sad_{h'}(s,a)\|\tilde\PP_{h'}(\cdot|s,a)-\PP_{h'}(\cdot|s,a)\|_1$. Thus,
\begin{align*}
    \|\tilde\sad-\sad\|_1
    &\leq\sum_{h=1}^H\sum_{h'=1}^{h-1}\sum_{s,a}\sad_{h'}(s,a)\|\tilde\PP_{h'}(\cdot|s,a)-\PP_{h'}(\cdot|s,a)\|_1\\
    &\leq H\sum_{h=1}^{H-1}\sum_{s,a}\sad_h(s,a)\|\tilde\PP_h(\cdot|s,a)-\PP_h(\cdot|s,a)\|_1.
\end{align*}
\end{proof}

\newpage

\section{PROOFS OF RESULTS IN SECTION~\ref{chapter: steering}}
\ThmUtilityRegretKnownModel*
\begin{proof}
    We abbreviate $\bmu^t=\bmu^t_{M^*}$. By Assumption~\ref{ass: adaptive regret}, $\sum_{t=1}^T\|\bmu^{t}-\sad^*\|_2^2=\sum_{t=1}^T\langle R^t(\bmu^{t}),\sad^*-\bmu^{t}\rangle\leq\adareg(T)$. Thus,
    \begin{align*}
        \sum_{t=1}^T\|\bmu^{t}-\sad^*\|_1^2
        \leq HSA\sum_{t=1}^T\|\bmu^{t}-\sad^*\|_2^2
        \leq HSA\adareg(T).
    \end{align*}
    By Assumption~\ref{assump:Lipschitz} and Jensen's inequality,
    \begin{align*}
        \max_\sad\sum_{t=1}^TU(\sad)-U(\bmu^{t})
        \leq L_U\cdot\sum_{t=1}^T\|\sad^*-\bmu^{t}\|_1
        \leq L_U\cdot\sqrt{HSAT\adareg(T)}.
    \end{align*}

    The steering cost can be bounded similarly.
    \begin{align*}
        \sum_{t=1}^T\langle R^t(\bmu^{t}),\bmu^{t}\rangle
        &=\sum_{t=1}^TH\|\sad^*-\bmu^{t}\|_\infty+\langle\sad^*-\bmu^{t},\bmu^{t}\rangle
        \leq2H\sum_{t=1}^T\|\sad^*-\bmu^{t}\|_\infty\\
        &\leq2H\sum_{t=1}^T\|\sad^*-\bmu^{t}\|_2
        \leq2H\sqrt{T\adareg(T)}.
    \end{align*}

    Given that $\adareg$ is sub-linear in $T$, we finish the proof.
\end{proof}

\ThmRegretPolicyReward*
\begin{proof}
    The bound for the steering gap can be shown by first using the $L_U$-Lipschitzness of $U$ and then applying Lem.~\ref{lem:general_version_density_gap} under Assump.~\ref{ass: adaptive regret}, where $\pi_*^t=\pi$ for all $t$.
    The calculation of the steering cost is the same as in the proof of Theorem \ref{thm: incentive designer regret} with $K=1$.
\end{proof}

\newpage

\section{PROOF OF THEOREM~\ref{thm: incentive designer regret}}\label{subsec: proof of ID regret thm}

\begin{lemma}\label{lem:general_version_density_gap}
    Let $(\pi_*^t)_{t=1}^T$ be a sequence of policies. We abbreviate $\bmu^t=\bmu^t_{M^*},\mu^{\pi_*^t}=\mu^{\pi_*^t}_{M^*}$. Then,
    \begin{align*}
        \frac1H\sum_{t=1}^T\|\bmu^{t} - \mu^{\pi_*^t}\|_1 \leq & \sum_{t=1}^T\sum_{h=1}^H\sum_{s\in\cS}\bmu^{t}_h(s)\Vert\bpi^{t}_h(\cdot|s)-\pi_{*,h}^t(\cdot|s)\Vert_1\\
        \leq&\sqrt{HSAT\sum_{t=1}^T\left\langle R_{\pi_*^t}(\bmu^{t}),\sad^{\pi_*^t}-\bmu^{t}\right\rangle},
    \end{align*}
    where $\bpi^{t}$ is the (population) policy which induces $\bmu^{t}=\sad^{\bpi^t}$.
\end{lemma}
\begin{proof}
    The first inequality follows from Lemma~\ref{lem: Lipschitzness of state-action density function}. We can write
    \begin{align*}
        &\sum_{t=1}^T\left\langle R_{\pi^{t}_*}(\bmu^{t}),\sad^{\pi^{t}_*}-\bmu^{t}\right\rangle
        =\sum_{t=1}^T\left\|(W^{\pi^{t}_*}-I)\bmu^{t}\right\|_2^2\\
        &=\sum_{t=1}^T\sum_{h,s,a}\left(\pi^{t}_{*,h}(a|s)\sum_{a'}\bmu^{t}_h(s,a')-\bmu^{t}_h(s,a)\right)^2\\
        &=\sum_{t=1}^T\sum_{h,s,a}(\bmu^{t}_h(s))^2\left(\pi^{t}_{*,h}(a|s)-\bpi^{t}_h(a|s)\right)^2\\
        &=\sum_{t=1}^T\sum_{h,s}(\bmu^{t}_h(s))^2\left\|\pi^{t}_{*,h}(\cdot|s)-\bpi^{t}_h(\cdot|s)\right\|_2^2.
    \end{align*}
    Furthermore, by Jensen's inequality,
    \begin{align*}
        &\sum_{t=1}^T\sum_{h,s}\bmu^{t}_h(s)\|\bpi^{t}_h(\cdot|s)-\pi^{t}_{*,h}(\cdot|s)\|_1
        \leq\sqrt{A}\sum_{t=1}^T\sum_{h,s}\bmu^{t}_h(s)\|\bpi^{t}_h(\cdot|s)-\pi^{t}_{*,h}(\cdot|s)\|_2\\
        &\leq\sqrt{HSAT\sum_{t=1}^T\sum_{h,s}(\bmu^{t}_h(s))^2\|\bpi^{t}_h(\cdot|s)-\pi^{t}_{*,h}(\cdot|s)\|_2^2}\\
        &=\sqrt{HSAT\sum_{t=1}^T\left\langle R_{\pi^{t}_*}(\bmu^{t}),\sad^{\pi^{t}_*}-\bmu^{t}\right\rangle}.
    \end{align*}
\end{proof}

\begin{lemma}\label{lem: transition estimate error bound}
    For any $0<\delta<1$, with probability at least $1-\delta$,
    \begin{align*}
        \|\PP_{\hat{M}^k,h}(\cdot|s,a)-\PP^*_h(\cdot|s,a)\|_1\leq2\epsilon_k(h,s,a)
    \end{align*}
    for all $t,h,s,a$, where $\epsilon_k(h,s,a):=\sqrt{\frac{2S\ln(THSA/\delta)}{\max\{1,N_k(h,s,a)\}}}$.
\end{lemma}

\begin{proof}
Since $\PP_{\hat{M}^k}\in\cP^{k}$, we have $\|\PP_{\hat{M}^k,h}(\cdot|s,a)-\bar\PP^{k}_h(\cdot|s,a)\|_1\leq\epsilon_k(h,s,a)$ for all $k,h,s,a$. By \eqref{eq: update transition confidence set} and Theorem 2.1 of \citet{Weissman2003InequalitiesFT},
\begin{align*}
    \Pr\left[\|\bar\PP^{k}_h(\cdot|s,a)-\PP^*_h(\cdot|s,a)\|_1>\epsilon\right]\leq(2^S-2)e^{-N_k(h,s,a)\epsilon^2/2}.
\end{align*}
Plugging in $\epsilon_k(h,s,a)$ for $\epsilon$ bounds this probability with $\delta/(THSA)$. The triangle inequality and a union bound over all $k,h,s,a$ imply the result.
\end{proof}

\begin{lemma}\label{lem: sum of transition estimate error bound}
    For any $0<\delta<1$ and respective $\epsilon_k(h,s,a)$,
    \begin{align*}
        \sum_{t=1}^T\sum_{h=1}^{H-1}\epsilon_{k(t)}(h,s^{t}_h,a^{t}_h)\leq3HS\sqrt{2\ln(THSA/\delta)AT}.
    \end{align*}
\end{lemma}

\begin{proof}
We can define $n_k(h,s,a):=\sum_{t=T_{k-1}+1}^{T_k}\II\{s^{t}_h=s,a^{t}_h=a\}$. Clearly, $N_k(h,s,a)=\sum_{k'<k}n_k(h,s,a)$. The condition in line 5 of the algorithm ensures that $n_k(h,s,a)\leq N_k(h,s,a)$ for all $k,h,s,a$. Thus, we can use Lemma 19 in \citet{jaksch10a} and Jensen's inequality,
\begin{align*}
    &\sum_{t=1}^T\sum_{h=1}^{H-1}\frac{1}{\sqrt{\max\{1,N_{k(t)}(h,s^{t}_h,a^{t}_h)\}}}
    =\sum_{k=1}^K\sum_{t=T_{k-1}+1}^{T_k}\sum_{h=1}^{H-1}\frac{1}{\sqrt{\max\{1,N_k(h,s^{t}_h,a^{t}_h)\}}}\\
    &=\sum_{k=1}^K\sum_{h=1}^{H-1}\sum_{s,a}\sum_{t=T_{k-1}+1}^{T_k}\frac{\II\{s^{t}_h=s,a^{t}_h=a\}}{\sqrt{\max\{1,N_k(h,s,a)\}}}
    =\sum_{k=1}^K\sum_{h=1}^{H-1}\sum_{s,a}\frac{n_k(h,s,a)}{\sqrt{\max\{1,N_k(h,s,a)\}}}\\
    &\leq3\sum_{h=1}^{H-1}\sum_{s,a}\sqrt{N_K(h,s,a)+n_K(h,s,a)}
    \leq3\sqrt{HSA\sum_{h=1}^{H-1}\sum_{s,a}(N_K(h,s,a)+n_K(h,s,a))}\\
    &=3\sqrt{HSA\cdot HT}=3H\sqrt{SAT}.
\end{align*}
Now, using the definition of $\epsilon_k(h,s,a)$,
\begin{align*}
    \sum_{t=1}^T\sum_{h=1}^{H-1}\epsilon_{k(t)}(h,s^{t}_h,a^{t}_h)
    &=\sum_{t=1}^T\sum_{h=1}^{H-1}\sqrt{\frac{2S\ln(THSA/\delta)}{\max\{1,N_{k(t)}(h,s^{t}_h,a^{t}_h)\}}}\\
    &\leq\sqrt{2S\ln(THSA/\delta)}\cdot3H\sqrt{SAT}=3HS\sqrt{2\ln(THSA/\delta)AT}.
\end{align*}
\end{proof}

\begin{restatable}{lemma}{LemConcBoundOfSad}\label{lem: concentration bound of state-action density}
    Let $(\bpi^{t})_{t=1}^T$ be the policy sequence of the population and $(\hat{M}^{k})_{k=1}^K$ the sequence of the corresponding model estimates. We abbreviate $\bmu^t=\bmu^t_{M^*},\hat\sad^{t}=\sad_{\hat{M}^{k(t)}}^{\bpi^{t}}$. With probability at least $1-2\delta$,
    \begin{align*}
        \sum_{t=1}^T\left\Vert\hat\sad^{t}-\bmu^{t}\right\Vert_1\leq12H^2S\sqrt{\ln(THSA/\delta)AT}.
    \end{align*}
\end{restatable}
\begin{proof}
    The proof is based on \citet{rosenberg2019online}.
    Let $(s^{t}_h,a^{t}_h)_{h=1}^H$ be the trajectory sampled in the $t$-th game. We define $\xi_k(h,s,a):=\|\PP_{\hat{M}^k,h}(\cdot|s,a)-\PP^*_h(\cdot|s,a)\|_1$. By Lemma~\ref{lem: state-action density difference},
\begin{align*}
    &\sum_{t=1}^T\left\|\hat\sad^{t}-\bmu^{t}\right\|_1
    \leq H\sum_{t=1}^T\sum_{h=1}^{H-1}\sum_{s,a}\bmu_h^{t}(s,a)\xi_{k(t)}(h,s,a)\\
    &=H\sum_{t=1}^T\sum_{h=1}^{H-1}\xi_{k(t)}(h,s^{t}_h,a^{t}_h)\\
    &~+H\sum_{t=1}^T\sum_{h=1}^{H-1}\underset{=:Y_t(h)}{\underbrace{\left(
        \sum_{s,a}\bmu^{t}_h(s,a)\xi_{k(t)}(h,s,a)-\sum_{s,a}\II\{s^{t}_h=s,a^{t}_h=a\}\xi_{k(t)}(h,s,a)\right)}},
\end{align*}
where $(Y_t(h))_t$ is a martingale difference sequence w.r.t.\ the trajectories sampled and with $|Y_t(h)|\leq\max_{s,a}\xi_{k(t)}(h,s,a)\leq2$.
In the following, we bound the first and second term above with high probability.

The first term can be bounded using Lemma~\ref{lem: transition estimate error bound}~and~\ref{lem: sum of transition estimate error bound}, such that we have, with probability at least $1-\delta$,
\begin{align*}
    H\sum_{t=1}^T\sum_{h=1}^{H-1}\xi_{k(t)}(h,s^{t}_h,a^{t}_h)
    \leq 2H\sum_{t=1}^T\sum_{h=1}^{H-1}\epsilon_{k(t)}(h,s^{t}_h,a^{t}_h)
    \leq 2H\cdot3H\sqrt{2S\ln(THSA/\delta)\cdot SAT}.
\end{align*}

By the Hoeffding-Azuma inequality, we have for a fixed $h$ that with probability at least $1-\delta/H$,
\begin{align*}
    \sum_{t=1}^TY_t(h)\leq2\sqrt{2T\ln(H/\delta)}.
\end{align*}
Thus, by the union bound over all $h$, the second term is at most $2H^2\sqrt{2T\ln(H/\delta)}$ with probability at least $1-\delta$.

Finally, by union bound over the events used to bound the first and second term, we have with probability at least $1-2\delta$ that
\begin{align*}
    \sum_{t=1}^T\left\|\hat\sad^{t}-\bmu^{t}\right\|_1
    &\leq2H^2\sqrt{2T\ln(H/\delta)}+6H^2\sqrt{2S\ln(THSA/\delta)\cdot SAT}\\
    &\leq12H^2S\sqrt{\ln(THSA/\delta)AT}.
\end{align*}

\end{proof}

\ThmSteeringGapNoReward*
\begin{proof}
We first establish the upper bound for steering gap and then investigate the steering cost.
\paragraph{Proof for Steering Gap}
We denote with $k(t)$ the episode index at the $t$-th game and denote $\pi^*=\argmax_\pi U(\sad_{M^*}^\pi)$. Furthermore, we abbreviate $\bmu^t=\bmu^t_{M^*},\sad_*^{k}=\sad_{M^*}^{\pi_*^{k}},\hat\sad^{t}=\sad_{\hat{M}^{k(t)}}^{\bpi^{t}}$ and $\hat\sad^{k}_*=\sad_{\hat{M}^{k}}^{\pi_*^{k}}$.
Consider a fixed $t$ and $k=k(t)$. We can decompose the steering gap term of round $t$ as follows:
\begin{align*}
    U(\sad^{\pi^*})-U(\bmu^{t})
    =\left(U(\sad^{\pi^*})-U(\hat\sad^{k}_*)\right)
    +\left(U(\hat\sad^{k}_*)-U(\bmu^{t})\right)
\end{align*}
The first term can be bounded by 0 using the optimism of the algorithm. We use the $L_U$-Lipschitzness of $U$ and the triangle inequality to further decompose the second term.
\begin{align*}
    U(\hat\sad^{k}_*)-U(\bmu^{t})
    \leq L_U\|\hat\sad^{k}_*-\bmu^{t}\|_1
    \leq L_U\|\hat\sad^{k}_*-\hat\sad^{t}\|_1+L_U\|\hat\sad^{t}-\bmu^{t}\|_1.
\end{align*}

Applying Lemma~\ref{lem: Lipschitzness of state-action density function}, we get
\begin{align*}
    &\|\hat\sad^{k}_*-\hat\sad^{t}\|_1
    \leq H\sum_{h,s}\hat\sd_h^{t}(s)\|\pi^{k}_{*,h}(\cdot|s)-\bpi^{t}_h(\cdot|s)\|_1\\
    &\leq H\sum_{h,s}\bmu_h^{t}(s)\cdot\|\pi^{k}_{*,h}(\cdot|s)-\bpi^{t}_h(\cdot|s)\|_1
    +H\underset{(*)}{\underbrace{\sum_{h,s}|\hat\sd_h^{t}(s)-\bmu_h^{t}(s)|\cdot\|\pi^{k}_{*,h}(\cdot|s)-\bpi^{t}_h(\cdot|s)\|_1}},
\end{align*}
where the second term can be bounded with
\begin{align*}
    (*)\leq2\sum_{h,s}|\hat\sd_h^{t}(s)-\bmu_h^{t}(s)|
    \leq2\sum_{h,s}\left|\sum_a\hat\sad_h^{t}(s,a)-\sum_a\bmu_h^{t}(s,a)\right|
    \leq2\|\hat\sad^{t}-\bmu^{t}\|_1.
\end{align*}

Putting it all together we now arrive at
\begin{align*}
    U(\sad^{\pi^*})-U(\bmu^{t})
    \leq L_UH\sum_{h,s}\bmu^{t}_h(s)\|\pi^{k}_{*,h}(\cdot|s)-\bpi^{t}_h(\cdot|s)\|_1
    +L_U(2H+1)\|\hat\sad^{t}-\bmu^{t}\|_1.
\end{align*}

By summing over $t$,
\begin{align*}
    \sum_{t=1}^TU(\sad^{\pi^*})-U(\bmu^{t})
    \leq L_UH\underset{\Delta_{\text{pop}}}{\underbrace{\sum_{t=1}^T\sum_{h,s}\bmu^t_h(s)\|\pi^{k(t)}_{*,h}(\cdot|s)-\bpi^{t}_h(\cdot|s)\|_1}}
    +L_U(2H+1)\underset{\Delta_{\text{est}}}{\underbrace{\sum_{t=1}^T\|\hat\sad^{t}-\bmu^{t}\|_1}}.
\end{align*}

Using Lemma~\ref{lem: concentration bound of state-action density}, the estimation error term $\Delta_{\text{est}}$ can bounded by $12H^2S\sqrt{\ln(THSA/\delta)AT}$ with probability at least $1-2\delta$.

To bound the population convergence term $\Delta_{\text{pop}}$, we can use Lemma~\ref{lem:general_version_density_gap}:
\begin{align*}
    \sum_{t=1}^T\sum_{h,s}\bmu^t_h(s)\|\pi^{k(t)}_{*,h}(\cdot|s)-\bpi^{t}_h(\cdot|s)\|_1
    \leq
    \vphantom{\underbrace{\sum_t^T}_{\texttt{AgentReg}}}
    \sqrt{
    \vphantom{\sum_t^T}
    \smash[b]{HSAT\! \underbrace{\sum_{t=1}^T\langle R_{\pi_*^{k(t)}}(\bmu^{t}),\sad_*^{k(t)}-\bmu^{t}\rangle}_{\texttt{AgentReg}}\,}
    }
\end{align*}

Furthermore, it can be easily seen that $\texttt{AgentReg}$ is
\begin{align*}
\sum_{t=1}^T\langle R_{\z}^{t}(\bmu^{t}),\sad_*^{k(t)}-\bmu^{t}\rangle
    &=\sum_{k=1}^K\underset{\leq\adareg(T)}{\underbrace{\sum_{t=T_{k-1}+1}^{T_k}\langle R_{\z}^{t}(\bmu^{t}),\sad_*^{k}-\bmu^{t}\rangle}}\leq K\cdot\adareg(T).
\end{align*}
Finally, to bound the number of episodes $K$, note that $K$ is also the number of times the condition in line 5 of the algorithm has been true. For each $(h,s,a)$, this condition can be true at most $\log_2T$ times. Thus, $K\leq HSA\log_2T$.

\paragraph{Proof for Steering Costs}
    Note that for any reward function $R$,
    \begin{align*}
        \langle R(\sad)+\|R(\sad)\|_\infty\bm1,\sad\rangle
        =H\|R(\sad)\|_\infty+\langle R(\sad),\sad\rangle
        \leq 2H\|R(\sad)\|_\infty.
    \end{align*}

    Let $\pi^*=\pi_*^{k}$ for some $k$. Recall that $R_{\pi^*}(\sad)=-((W^{\pi^*}-I)\sad)^\top(W^{\pi^*}-I)$. By looking at the definition of $W^{\pi^*}$ in (\ref{eq: policy reward matrix}), we see that
    \begin{align*}
        \left\|(W^{\pi^*}-I)^\top\right\|_\infty=\max_{h,s,a}\sum_{a'\neq a}|\pi_h(a'|s)|+|\pi_h(a|s)-1|\leq2,
    \end{align*}
    where the $\|\cdot\|_\infty$-matrix norm is defined as $\|M\|_\infty=\max_i\sum_j|M_{ij}|$.
    Using this, we can bound
    \begin{align*}
        \left\|R_{\pi^*}(\sad)\right\|_\infty
        &=\left\|((W^{\pi^*}-I)\sad)^\top(W^{\pi^*}-I)\right\|_\infty=\left\|(W^{\pi^*}-I)^\top(W^{\pi^*}-I)\sad\right\|_\infty\\
        &\leq\left\|(W^{\pi^*}-I)^\top\right\|_\infty\cdot\left\|(W^{\pi^*}-I)\sad\right\|_\infty
        \leq2\left\|(W^{\pi^*}-I)\sad\right\|_2.
    \end{align*}

    Finally, using Jensen's inequality and the fact that the agent regret is bounded by $K\adareg(T)$, our steering cost can be bounded by
    \begin{align*}
        &\sum_{t=1}^T\left\langle R_{\z}^{t}(\bmu^{t}),\bmu^{t}\right\rangle
        =\sum_{t=1}^T\left\langle R_{\pi_*^{k(t)}}(\bmu^{t})+\|R_{\pi_*^{k(t)}}(\bmu^{t})\|_\infty\bm1,\bmu^{t}\right\rangle\\
        &\leq4H\sum_{t=1}^T\left\|(W^{\pi^{k(t)}_*}-I)\bmu^{t}\right\|_2
        \leq4H\sqrt{T\sum_{t=1}^T\left\|(W^{\pi^{k(t)}_*}-I)\bmu^{t}\right\|_2^2}\\
        &\leq4H\sqrt{T\sum_{t=1}^T\langle R_{\z}^{t}(\bmu^{t}),\sad_*^{k(t)}-\bmu^{t}\rangle}
        \leq4H\sqrt{TK\adareg(T)}
    \end{align*}
\end{proof}

\newpage
\section{ELUDER DIMENSION}\label{appx:eluder}

\subsection{Example Function Classes}\label{appx:example_function_class}

Here, we list some bounds of the eluder dimension for different function classes that are commonly considered. We see that in all these cases, the eluder dimension can be bounded logarithmically in $T$, if $\epsilon=T^{-1}$.

\begin{proposition}[Linear functions, \citet{Eluder}]
    Let $\cF=\{f|f(x)=\theta^\top\phi(x),\theta\in\RR^d,\|\theta\|_2\leq C_\theta,\|\phi(x)\|_2\leq C_\phi\}$.
    \begin{align*}
        \dim_E(\cF,\epsilon)\leq 3d\frac{e}{e-1}\ln\left(3+3\left(\frac{2C_\theta}{\epsilon}\right)^2\right)+1.
    \end{align*}
\end{proposition}

\begin{proposition}[Quadratic functions, \citet{EluderVector}]
    Let $\cF=\{f|f(x)=\phi(x)^\top\theta\phi(x),\theta\in\RR^{p\times p},\phi\in\RR^p,\|\theta\|_2\leq C_\theta,\|\phi\|_2\leq C_\phi\}$.
    \begin{align*}
        \dim_E(\cF,\epsilon)\leq p(4p-1)\frac{e}{e-1}\log\left(\left(1+\left(\frac{2pC_\phi^2C_\theta}{\epsilon}\right)^2\right)(4p-1)\right)+1.
    \end{align*}
\end{proposition}

\begin{proposition}[Generalized linear functions, \citet{Eluder}]
    Let $g$ be strictly increasing, differentiable and have derivatives bounded in $[\underline{h},\overline{h}]$ with $\overline{h}>\underline{h}>0$. Let $r=\overline{h}/\underline{h}$ and $\cF=\{f|f(x)=g(\theta^\top\phi(x)),\theta\in\RR^d,\|\theta\|_2\leq C_\theta,\|\phi\|_2\leq C_\phi\}$.
    \begin{align*}
        \dim_E(\cF,\epsilon)\leq3dr^2\frac{e}{e-1}\log\left(3r^2+3r^2\left(\frac{2C_\theta\overline{h}}{\epsilon}\right)^2\right)+1.
    \end{align*}
\end{proposition}

\begin{remark}[Bounding $\beta_T$]\label{rem: bounding beta}
If we assume that the functions in $\cR$ are parametrized by parameters in some set $\Theta\subset\RR^d$ with constant diameter and the functions are $L$-Lipschitz in that parameter, we have $N(\cR,\alpha,\Vert\cdot\Vert_\infty)\leq N(\Theta,\alpha/L,\Vert\cdot\Vert_\infty)\leq\left(1+\cO(L/\alpha)\right)^d$.
Then, we might choose $\alpha=T^{-1}$ such that
\begin{align*}
    \beta_T=8\sigma^2\log(N(\cR,\alpha,\Vert\cdot\Vert_\infty)/\delta)
    +2\alpha T(8r_{\max}+\sqrt{8\sigma^2\ln(4T^2/\delta)})
\end{align*}
can also be bounded logarithmically in $T$.
\end{remark}

\subsection{Bounding The Width Of The Confidence Set}

\paragraph{Notations and Definitions}
Here, we introduce some notation used in this section.
We define the width function $w_\cF(x)=\sup_{\underline{f},\overline{f}\in\cF}|\underline{f}(x)-\overline{f}(x)|$.
Throughout this section, we use the notation $x_{H_t+h}$ with $H_t=(t-1)H$ to describe elements of a sequence $x_1,...,x_{HT}$. The idea behind it is that we can later define $x_{H_t+h}=(h,s_h^{t},a_h^{t},\bmu^{t}_{M^*,h})$ and apply the results in this section to our setting.
Furthermore, for any function $g$ we write $\|g\|_{2,E_t}^2=\sum_{i=1}^{t-1}\sum_{h=1}^Hg^2(x_{H_t+h})$.

\begin{lemma}[Proposition 3 of \citet{Eluder}]\label{lem: eluder bound number of large widths}
    If $(\beta_t)_{t\in\NN}$ is a positive non-decreasing sequence, $(\hat f_t)_t$ some function sequence and $\cF_t:=\{f\in\cF:\|f-\hat f_t\|_{2,E_t}\leq\sqrt{\beta_t}\}$ then with probability 1, for all $T\in\NN$,
    \begin{align*}
        \sum_{t=1}^T\sum_{h=1}^H\II\{w_{\cF_t}(x_{H_t+h})>\epsilon\}
        \leq\left(\frac{4\beta_T}{\epsilon^2}+H\right)\dim_E(\cF,\epsilon)
    \end{align*}
    for all $T\in\NN$ and $\epsilon>0$.
\end{lemma}
\begin{proof}
    First we show that for any $\tau=H_t+h<TH$, if $w_{\cF_t}(x_\tau)>\epsilon$ then $x_\tau$ is $(\cF,\epsilon)$-dependent on fewer than $4\beta_T/\epsilon^2$ disjoint subsequences of $(x_1,...,x_{H_t})$. Suppose $w_{\cF_t}(x_\tau)>\epsilon$. Then, there are $f,\tilde f\in\cF_t$ such that $|f(x_\tau)-\tilde f(x_\tau)|>\epsilon$. Furthermore, let $(x_{i_1},...,x_{i_k})$ be a subsequence of $(x_1,...,x_{H_t})$ on which $x_\tau$ is $(\cF,\epsilon)$-dependent. This implies, by definition, that $\sum_{j=1}^k(f(x_{i_j})-\tilde f(x_{i_j}))^2>\epsilon^2$. If $x_\tau$ is $(\cF,\epsilon)$-dependent on $K$ disjoint subsequences of $(x_1,...,x_{H_t})$ then we must have
    \begin{align*}
        \|f-\tilde f\|_{2,E_t}^2=\sum_{i=1}^{t-1}\sum_{h=1}^H(f(x_{H_i+h})-\tilde f(x_{H_i+h}))^2
        \geq\sum_{l=1}^K\sum_{j=1}^{k_l}(f(x_{i^l_j})-\tilde f(x_{i^l_j}))^2
        >K\epsilon^2.
    \end{align*}
    By the triangle inequality, $\|f-\tilde f\|_{2,E_t}\leq\|f-\hat f_t\|_{2,E_t}+\|\tilde f-\hat f_t\|_{2,E_t}\leq2\sqrt{\beta_t}\leq2\sqrt{\beta_T}$. Combining these two inequalities, we get $K<4\beta_T/\epsilon^2$.

    Next, we show that in any sequence $(y_1,...,y_l)$ there is an element $y_j$ which is $(\cF,\epsilon)$-dependent on at least $l/d-1$ disjoint subsequences of $(y_1,...,y_{j-1})$, where $d=\dim_E(\cF,\epsilon)$. Let $K$ be an integer with $Kd+1\leq l\leq Kd+d$. We will construct $K$ disjoint subsequences $B_1,...,B_K$. First, $B_i=(y_i)$ for all $i\in[K]$. If $y_{K+1}$ is already $(\cF,\epsilon)$-dependent on $B_1,...,B_K$, we are done. Otherwise, select a $B_i$ of which $y_{K+1}$ is $(\cF,\epsilon)$-independent and append $y_{K+1}$ to $B_i$. We repeat this for $y_{K+2},y_{K+3},...$ until we find $y_j$ that is $(\cF,\epsilon)$-dependent on each subsequence or until we have reached $y_l$. In the latter case, each element of a subsequence $B_i$ is independent of its predecessors and hence $|B_i|=d$. Then, $y_l$ must be $(\cF,\epsilon)$-dependent on each subsequence, by definition of the eluder dimension. In both cases we find an element in $(y_1,...,y_l)$ that is $(\cF,\epsilon)$-dependent on $K\geq t/d-1$ disjoint subsequences.

    Finally, let $(y_1,...,y_l)=(x_{i_1},...,x_{i_l})$ be a subsequence of $(x_1,...,x_{TH})$ consisting of all elements $x_{H_t+h}$ for which $w_{\cF_t}(x_{H_t+h})>\epsilon$. From before, we know there is some $y_j$ that is $(\cF,\epsilon)$-dependent on at least $l/d-1$ disjoint subsequences of $(y_1,...,y_{j-1})$. Let $t,h$ be such that $y_j=x_{H_t+h}$. Note that in $(y_1,...,y_{j-1})$ there are at most $H-1$ elements $y_i=x_{H_t+h'}$ for some $h'<h$. From this follows that $y_j=x_{H_t+h}$ is $(\cF,\epsilon)$-dependent on at least $l/d-1-(H-1)=l/d-H$ disjoint subsequences of $(y_1,...,y_{j-H})\subseteq(x_1,...,x_{H_t})$. Now, as we have also shown, $x_{H_t+h}$ is $(\cF,\epsilon)$-dependent on fewer than $4\beta_T/\epsilon^2$ disjoint subsequences of $(x_1,...,x_{H_t})$. Combining these two bounds, we get $l/d-H\leq4\beta_T/\epsilon^2$, and therefore $l\leq(4\beta_T/\epsilon^2+H)d$.
\end{proof}

\begin{restatable}[Variant of Lemma 2 in \citet{Eluder}]{lemma}{LemEluderConfSetBound}\label{lem: eluder confidence set bound}
    Let $(\beta_t)_{t\in\NN}$ be a positive non-decreasing sequence, $(\hat f_t)_t$ some function sequence and $\cF_t:=\{f\in\cF:\Vert f-\hat f_t\Vert_{2,E_t}\leq\sqrt{\beta_t}\}$. Let $w_{\cF}(x)\leq C$ for all $x$. Then, for all $T\in\NN$ and $\epsilon>0$,
    \begin{align*}
        \sum_{t=1}^T\sum_{h=1}^Hw_{\cF_t}(x_{H_t+h})
        \leq\epsilon HT+CH\dim_E(\cF,\epsilon)\\
        +4\sqrt{\beta_TH\dim_E(\cF,\epsilon)T}.
    \end{align*}
\end{restatable}
\begin{proof}
    We abbreviate $w_{H_t+h}=w_{\cF_t}(x_{H_t+h})$ and $d=\dim_E(\cF,\epsilon)$. Let $w_{i_1}\geq...\geq w_{i_{HT}}$. Using this ordering of the sequence, $w_{i_k}>\epsilon$ implies that $\sum_{j=1}^T\II\{w_j>\epsilon\}\geq k$. By Lemma~\ref{lem: eluder bound number of large widths}, this would mean $k\leq(4\beta_T/\epsilon^2+H)d$ or, equivalently, $\epsilon<\sqrt{4\beta_Td/(k-Hd)}$. Now, since $w_{i_k}>\epsilon$ implies $\epsilon<\sqrt{4\beta_Td/(k-Hd)}$, this means that $w_{i_k}<\sqrt{4\beta_Td/(k-Hd)}$.

    In the following, we bound the first and largest widths $w_{i_1},...,w_{i_{Hd}}$ by $C$ and the remaining widths (larger than $\epsilon$) by the previously established bound.
    \begin{align*}
        \sum_{t=1}^T\sum_{h=1}^Hw_{H_t+h}
        &=\sum_{k=1}^{HT}\II\{w_k\leq\epsilon\}w_k+\sum_{k=1}^{HT}\II\{w_k>\epsilon\}w_k
        \leq\epsilon HT+\sum_{k=1}^{HT}\II\{w_k>\epsilon\}w_k\\
        &\leq\epsilon HT+HdC+\sum_{k=Hd+1}^{HT}\II\{w_{i_k}>\epsilon\}w_{k_t}\\
        &\leq\epsilon HT+HdC+\sum_{k=Hd+1}^{HT}\sqrt{4\beta_Td/(k-Hd)}\\
        &\leq\epsilon HT+HdC+\sqrt{4d\beta_T}\int_0^{HT}\frac{1}{\sqrt{x}}dx\\
        &=\epsilon HT+HdC+4\sqrt{d\beta_THT}
    \end{align*}
\end{proof}

\newpage

\section{PROOF OF THEOREM~\ref{thm: incentive designer regret unknown reward}}\label{appx:proof_non_zero_rew}
\subsection{Algorithm Details}\label{appx:alg_nonzero_reward}
We present our full algorithm for the unknown reward setting in Alg.~\ref{alg: incentive design unknown reward}.
\begin{algorithm}
\caption{Steering reward design for Scenario 2}
\label{alg: incentive design unknown reward}
\begin{algorithmic}[1]
    \State Initialize $\cP^{1}:=$ set of all possible transition functions, $\pi_*^{1}$ (arbitrarily), $k=1,T_0=0$.
    \For {$t=1,...,T$}
        \State Update $\hat\cR^{t}$ as in (\ref{eq: reward confidence set}).
        \State Choose $R_{\nz}^{t}$ as in (\ref{eq: reward modification}).
        \State Agents play $t$-th game with $r^*+R_{\nz}^{t}$.
        \State Obtain trajectory $((s^{t}_h,a^{t}_h,r^t_h))_{h=1}^H$.
        \If {$\exists(h,s,a),~s.t.~n_k(h,s,a)\geq N_k(h,s,a)$}\label{line:if-condition}
            \State Update $\cP^{k+1}$ as in (\ref{eq: update transition confidence set}).
            \State $T_k\gets t$; $k\gets k+1$.
            \State
            $
            \pi_*^{k},\hat{M}^{k}\gets \argmax_{\pi\in\Pi,\hat{M}:\PP_{\hat M}\in\cP^{k}} U(\sad^\pi_{\hat{M}}).
            $\label{line:opt_policy}
        \EndIf
    \EndFor
\end{algorithmic}
\end{algorithm}

\subsection{Missing Proofs}
\begin{lemma}[Proposition 2 in \citet{Eluder}]\label{lem: true unknown reward in confidence set}
    Let $N(\cR,\alpha,\Vert\cdot\Vert_\infty)$ be the $\alpha$-covering number of $\cR$ w.r.t.\ the $\Vert\cdot\Vert_\infty$-norm. Let $\delta>0,\alpha>0$, and for each $t$, $\beta_t=8\sigma^2\log(N(\cR,\alpha,\Vert\cdot\Vert_\infty)/\delta)
        +2\alpha t(8r_{\max}+\sqrt{8\sigma^2\ln(4t^2/\delta)})$.
    With probability at least $1-2\delta$, $r^*\in\bigcap_{t=1}^\infty\hat\cR^{t}$.
\end{lemma}

\begin{lemma}\label{lem: expected confidence width bound}
    We abbreviate $\bmu^t=\bmu^t_{M^*}$. With probability at least $1-\delta$,
    \begin{align*}
        \sum_{t=1}^T\langle w_{\hat\cR^{t}}(\bmu^{t}),\bmu^{t}\rangle
        \leq3\sum_{t=1}^T\sum_{h=1}^Hw_{\hat\cR^{t}}(h,s_h^{t},a_h^{t},\bmu^{t})+r_{\max}H\ln(1/\delta).
    \end{align*}
\end{lemma}
\begin{proof}
    Note that $\langle w_{\hat\cR^{t}}(\bmu^{t}),\bmu^{t}\rangle=\EE_{(s_h,a_h)_{h=1}^H\sim\bmu^{t}}[\sum_{h=1}^Hw_{\hat\cR^{t}}(h,s_h,a_h,\bmu^{t})]=:Y_t$. Recall that $(s_h^{t},a_h^{t})_h\sim\bmu^{t}$ are the trajectories we gather from the population at step $t$. Therefore, we can define $X_t:=\sum_{h=1}^Hw_{\hat\cR^{t}}(h,s_h^{t},a_h^{t},\bmu^{t})$ with $\EE[X_t|\bmu^{t}]=Y_t$. By the assumption that $r^*$ is bounded in $[0,r_{\max}]$, we have that $w_{\hat\cR}(h,s,a,\sad)\leq r_{\max}$ for any $\sad,h,s,a$ and $\hat\cR\subseteq\cR$. Therefore, $0\leq X_t\leq r_{\max}H$. A direct application of Lemma D.4 from \citet{huang_statistical_2023} shows that with probability at least $1-\delta$,
    \begin{align*}
        \sum_{t=1}^TY_t\leq3\sum_{t=1}^TX_t+r_{\max}H\ln\frac1\delta.
    \end{align*}
\end{proof}

\begin{restatable}{lemma}{LemUnknownRewardAgentRegret}\label{lem: unknown reward agent regret}
    We abbreviate $\sad_*^{k}=\sad_{M^*}^{\pi_*^{k}},\bmu^t=\bmu^t_{M^*}$. If the true $r^*$ is contained in all $\hat\cR^{t}$, then, with probability at least $1-\delta$,
    \begin{align*}
        \sum_{t=1}^T\langle R_{\pi_*^{k(t)}}(\bmu^{t}),\sad_*^{k(t)}-\bmu^{t}\rangle
        \leq\sum_{t=1}^T\langle r^*(\bmu^{t})+R_{\nz}^{t}(\bmu^{t}),\sad_*^{k(t)}-\bmu^{t}\rangle
        +6\sum_{t=1}^T\sum_{h=1}^Hw_{\hat\cR^{t}}(h,s_h^{t},a_h^{t},\bmu^{t})+2r_{\max}H\ln\frac1\delta.
    \end{align*}
\end{restatable}
\begin{proof}
    Let $t\in[T]$ and $k=k(t)$. By Eq.~\eqref{eq: reward modification},
    \begin{align*}
        \langle R_{\pi_*^{k}}(\bmu^{t}),\sad_*^{k}-\bmu^{t}\rangle
        =\langle r^*(\bmu^{t})+R_{\nz}^{t}(\bmu^{t}),\sad_*^{k}-\bmu^{t}\rangle
        +\langle\bar r^{t}(\bmu^{t})-r^*(\bmu^{t})-w_{\hat\cR^{t}}(\bmu^{t}),\sad_*^{k}-\bmu^{t}\rangle.
    \end{align*}
    With that, we have already separated out the first term (agent regret). Using the assumption that $r^*\in\hat\cR^{t}$ for all $t$, we can bound the second term as follows.
    \begin{align*}
        &\langle\bar r^{t}(\bmu^{t})-r^*(\bmu^{t})-w_{\hat\cR^{t}}(\bmu^{t}),\sad_*^{k}-\bmu^{t}\rangle\\
        &=\langle r^*(\bmu^{t})-\bar r^{t}(\bmu^{t})+w_{\hat\cR^{t}}(\bmu^{t}),\bmu^{t}\rangle
        +\langle\underset{\leq w_{\hat\cR^{t}}(\bmu^{t})}{\underbrace{\bar r^{t}(\bmu^{t})-r^*(\bmu^{t})}}-w_{\hat\cR^{t}}(\bmu^{t}),\sad^{k}_*\rangle\\
        &\leq\langle\underset{\leq w_{\hat\cR^{t}}(\bmu^{t})}{\underbrace{r^*(\bmu^{t})-\bar r^{t}(\bmu^{t})}}+w_{\hat\cR^{t}}(\bmu^{t}),\bmu^{t}\rangle
        \leq2\langle w_{\hat\cR^{t}}(\bmu^{t}),\bmu^{t}\rangle
    \end{align*}
    Finally, we can bound $\sum_{t=1}^T\langle w_{\hat\cR^{t}}(\bmu^{t}),\bmu^{t}\rangle$ using Lemma~\ref{lem: expected confidence width bound}, which implies the result.
\end{proof}

\ThmUtilityRegretUnknownReward*
\begin{proof}
    We abbreviate $\sad_*^{k}=\sad_{M^*}^{\pi_*^{k}},\bmu^t=\bmu^t_{M^*}$. We can use the exact same arguments as in the proof of Theorem~\ref{thm: incentive designer regret}, up until the point where we have to bound
    \begin{align*}
        \texttt{AgentReg}
        =\sum_{t=1}^T\langle R_{\pi_*^{k(t)}}(\bmu^{t}),\sad_*^{k(t)}-\bmu^{t}\rangle.
    \end{align*}
    Combining Lemma~\ref{lem: true unknown reward in confidence set} and Lemma~\ref{lem: unknown reward agent regret} we have with probability at least $1-3\delta$ that
    \begin{align*}
        &\sum_{t=1}^T\langle R_{\pi_*^{k(t)}}(\bmu^{t}),\sad_*^{k(t)}-\bmu^{t}\rangle
        \leq
        \underset{\texttt{NewAgentReg}}{\underbrace{\sum_{t=1}^T\langle r^*(\bmu^t)+R_{\nz}^{t}(\bmu^{t}),\sad_*^{k(t)}-\bmu^{t}\rangle}}
        +6\sum_{t=1}^T\sum_{h=1}^Hw_{\hat\cR^{t}}(h,s_h^{t},a_h^{t},\bmu^{t})+2r_{\max}H\ln\frac1\delta.
    \end{align*}
    Now, we summarize $x_{H_t+h}=(h,s_h^{t},a_h^{t},\bmu^{t}_h)$, where $H_t=H(t-1)$, and with slight abuse of notation, we rewrite $\hat{r}(x_{H_t+h}) = \hat{r}_h(s_h^{t},a_h^{t},\bmu^{t}_h)$. With this rewriting of notation, we can apply Lemma~\ref{lem: eluder confidence set bound} with $\epsilon=T^{-1}$ to show that
    \begin{align*}
        \sum_{t=1}^T\sum_{h=1}^Hw_{\hat\cR^{t}}(h,s_h^{t},a_h^{t},\bmu^{t})
        \leq
        H+r_{\max}H\dim_E(\cR,T^{-1})+4\sqrt{\beta_TH\dim_E(\cR,T^{-1})T}.
    \end{align*}
    Combining the with the previous results, we get that with probability at least $1-3\delta$,
    \begin{align*}
        \sum_{t=1}^T\langle R_{\pi_*^{k(t)}}(\bmu^{t}),\sad_*^{k(t)}-\bmu^{t}\rangle
        &\leq
        \texttt{NewAgentReg}\\
        &\quad+\underset{=:D}{\underbrace{6\left(H+r_{\max}H\dim_E(\cR,T^{-1})+4\sqrt{\beta_TH\dim_E(\cR,T^{-1})T}\right)+2r_{\max}H\ln\frac1\delta}}.
    \end{align*}
    The new agent regret term \texttt{NewAgentReg} can be bounded in the same way as in the proof of Theorem~\ref{thm: incentive designer regret}:
    \begin{align*}
        \sum_{t=1}^T\langle r^*(\bmu^{t})+R_{\nz}^{t}(\bmu^{t}),\sad_*^{k(t)}-\bmu^{t}\rangle
        \leq\sum_{k=1}^K\sum_{t=T_{k-1}+1}^{T_k}\langle r^*(\bmu^{t})+R_{\nz}^{t}(\bmu^{t}),\sad_*^{k}-\bmu^{t}\rangle
        \leq K\adareg(T).
    \end{align*}

    For the steering cost, we have for any $\sad\in\Psi_M$,
    \begin{align*}
        C(\sad,R_{\nz}^{t})-C(\sad,r_{\max}\bm1-r^*)
        &=\langle r^*(\sad)-\bar r^{t}(\sad)+w_{\hat\cR^{t}}(\sad)
        +R_{\pi_*^{k(t)}}(\sad)+\|R_{\pi_*^{k(t)}}(\sad)\|_\infty\bm1,\sad\rangle\\
        &\leq2\langle w_{\hat\cR^{t}}(\sad),\sad\rangle
        +\langle R_{\pi_*^{k(t)}}(\sad)+\|R_{\pi_*^{k(t)}}(\sad)\|_\infty\bm1,\sad\rangle
    \end{align*}
    Then, summing over $t=1,...,T$,
    \begin{align*}
        C_T(\{\bmu^{t},R_{\nz}^{t}-(r_{\max}\bm1-r^*)\}_{t=1}^T)
        \leq2\sum_{t=1}^T\langle w_{\hat\cR^{t}}(\bmu^{t}),\bmu^{t}\rangle
        +\sum_{t=1}^T\langle R_{\pi_*^{k(t)}}(\bmu^{t})+\|R_{\pi_*^{k(t)}}(\bmu^{t})\|_\infty\bm1,\bmu^{t}\rangle.
    \end{align*}
    Using Lemma~\ref{lem: expected confidence width bound}, we can bound the first term by $2(3\sum_{t=1}^T\sum_{h=1}^Hw_{\hat\cR^{t}}(x_{H_t+h})+r_{\max}H\ln(1/\delta))$ with probability at least $1-\delta$. Using Lemma~\ref{lem: eluder confidence set bound} with $\epsilon=T^{-1}$, we can further bound this by $D=2r_{\max}H\ln(1/\delta)+6(H+r_{\max}H\dim_E(\cF,T^{-1})+4\sqrt{\beta_TH\dim_E(\cR,T^{-1})T})$.
    From the steering cost bound in Thm.~\ref{thm: incentive designer regret} follows that the second term is bounded by
    \begin{align*}
        4H\sqrt{T\sum_{t=1}^T\langle R_{\pi_*^{k(t)}}(\bmu^{t}),\sad_*^{k(t)}-\bmu^{t}\rangle},
    \end{align*}
    which is at most $4H\sqrt{T(K\adareg(T)+D)}$, as we have already shown in this proof.

    Lastly, we have to discuss the asymptotic bound for $D$. The term $D$ is dependent on the Eluder dimension of $\cR$ and $\beta_T$. In Appendix~\ref{appx:example_function_class}, we show several common function classes with $\dim_E(\cR,T^{-1})\in\tilde\cO(1)$. Furthermore, if we assume that the functions in $\cR$ are parametrized by parameters in some set $\Theta\subset\RR^d$ with constant diameter and $L$-Lipschitz in that parameter, we have $N(\cR,\alpha,\Vert\cdot\Vert_\infty)\leq N(\Theta,\alpha/L,\Vert\cdot\Vert_\infty)\leq\left(1+\cO(L/\alpha)\right)^d$
Then, we might choose $\alpha=T^{-1}$ such that $\beta_T$ can also be bounded logarithmically in $T$. In the cases where the Eluder dimension and $\beta_T$ are in $\tilde\cO(1)$, we have $D\in\tilde\cO(\sqrt{T})$ (ignoring other factors).
\end{proof}

\section{EXTENSION TO UNKNOWN UTILITY FUNCTION}\label{appx:unknown utility}
In this section, we generalize our previous results to the setting where the mediator does not have prior knowledge of the utility function $U$.
We consider non-zero intrinsic reward setting, i.e., Scenario 2 described in Section \ref{sec:objectives}.
Note that the results for Scenario 1 can be directly derived by setting $r^* = 0$.

\paragraph{Motivation for Unknown Utility Setting}
This setting makes sense, especially when $U$ partially depends on the agents' intrinsic rewards $r^*$.
As a motivating example, in financial markets, the government (mediator) gains benefits (utility $U$) from not only the impact on the society by the desired behaviors of the companies (the agents), but also the tax paid by them, which is directly related to the rewards $r^*$ received by agents.\footnote{Another way to interpret this scenario is that the mediator's utility $U = \alpha U_{\text{mediator}} + (1-\alpha) U_{\text{agents}; r^*}$ can be decomposed to a known function $U_{\text{mediator}}$ representing its intrinsic utility, and another unknown part $U_{\text{agents}; r^*}$, which reflects the agents' interests and depends on $r^*$. Here $\alpha$ serves as a parameter to trade-off the interests between two parties.
}
Due to the lack of knowledge of $r^*$, $U$ should only be partially revealed to the mediator.
This restricts the applicability of our methods to this setting.
However, if $U$ is unknown, we might infer it, for example, by estimating the true reward functions $r^*$ through the online interaction with the agents. We can also generalize this setting as follows.

We consider a general setting, where the mediator does not have prior knowledge on $U$, but it can observe samples from $U$, perturbed by $\sigma_U$-sub-Gaussian noise, and get access to a function class $\cU$ which contains $U$ and whose functions are bounded in $[0,U_{\max}]$.

\subsection{Algorithm}

We can use the standard technique described in \citet{Eluder} to handle this case.
We define
\begin{align}
    \bar{U}^k&=\argmin_{\hat{U}\in\cU}\sum_{t=1}^{T_k}(\hat{U}(\bmu^t)-U(\bmu^t))^2,\\
    \hat{\cU}^k&=\left\{\hat{U}\in\cU:\|\hat{U}-\bar{U}^k\|_{2,E_{T_k}}^2\leq\beta_{k}^U\right\},\label{eq: utility confidence set}
\end{align}
where $\beta_k^U:=8\sigma_U^2\log(N(\cU,\alpha,\Vert\cdot\Vert_\infty)/\delta)+2\alpha k(8U_{\max}+\sqrt{8\sigma_U^2\ln(4k^2/\delta)})$ and, e.g., $\alpha=T^{-1}$.

\begin{algorithm}
\caption{Steering reward design for Scenario 2 and unknown utility}
\label{alg: incentive design unknown reward and utility}
\begin{algorithmic}[1]
    \State Initialize $\cP^{1}:=$ set of all possible transition functions, $\pi_*^{1}$ (arbitrarily), $k=1,T_0=0$.
    \For {$t=1,...,T$}
        \State Update $\hat\cR^{t}$ as in (\ref{eq: reward confidence set}).
        \State Choose $R_{\nz}^{t}$ as in (\ref{eq: reward modification}).
        \State Agents play $t$-th game with $r^*+R_{\nz}^{t}$.
        \State Obtain trajectory $((s^{t}_h,a^{t}_h,r^t_h))_{h=1}^H$.
        \If {$\exists(h,s,a),~s.t.~n_k(h,s,a)\geq N_k(h,s,a)$ or $t-T_{k-1}\geq T_{\mathit{epoch}}$}\label{line:if-condition unknown utility}
            \State Update $\cP^{k+1}$ as in (\ref{eq: update transition confidence set}).
            \State $T_k\gets t$; $k\gets k+1$.
            \State Compute $\hat\cU^k$ as in \eqref{eq: utility confidence set}.\label{line: utility confidence set}
            \State
            $
            \hat{U}^{k},\pi_*^{k},\hat{M}^{k}\gets\argmax_{\hat{U}\in\hat{\cU}^{k},\pi\in\Pi,\hat{M}:\PP_{\hat M}\in\cP^{k}} \hat{U}(\sad^\pi_{\hat{M}}).
            $\label{line:opt_policy unknown utility}
        \EndIf
    \EndFor
\end{algorithmic}
\end{algorithm}

Algorithm~\ref{alg: incentive design unknown reward and utility} differs from Algorithm~\ref{alg: incentive design unknown reward} in the if-condition in line~\ref{line:if-condition unknown utility} as well as in lines \ref{line: utility confidence set} and \ref{line:opt_policy unknown utility}.

The if-condition in line \ref{line:if-condition} now includes the case $t-T_{k-1}\geq T_{\mathit{epoch}}$, where $T_{\mathit{epoch}}$ will be chosen later.
We need this to guarantee $T_k-T_{k-1}\leq T_{\mathit{epoch}}$ for all $k$ and thereby bound the estimation error of the utility function estimate. Intuitively, we need to keep the estimates of $U$ somewhat up to date to be able to bound the estimation error. Meanwhile, we cannot update the estimate in each round (or too often) since then we would also have to change $\pi_*^k$ in each round, which would lead to $K=T$.

Since we also need to estimate the utility function, we changed line~\ref{line:opt_policy unknown utility} to also compute an optimistic estimate of the utility using the definition in \eqref{eq: utility confidence set}.

\subsection{Analysis}

\begin{theorem}\label{thm: incentive designer regret unknown reward unknown utility}
    Under Assump.~\ref{assump:realizability},~\ref{ass: adaptive regret} and~\ref{assump:Lipschitz}, if we run Alg.~\ref{alg: incentive design unknown reward and utility} with $0<\delta<1$, then with probability at least $1-8\delta$, $K\leq T^{1/6}+HSA\log_2T$, and
    \begin{align*}
        \Delta_T(\{\bmu^{t}_{M^*}\}_{t=1}^T)
        &\leq L_U\sqrt{H^3SAT(K\adareg(T)+D)}+36L_UH^3S\sqrt{AT\ln(THSA/\delta)}\\
        &+\cO\left(T^{5/6}U_{\max}\dim_E(\cU,T^{-1})+\sqrt{\beta_K^U\dim_E(\cU,T^{-1})T}\right),\\
        C_T(\{\bmu^{t}_{M^*},R_{\nz}^{t} - & (r_{\max}\cdot \mathbf{1} - r^*)\}_{t=1}^T)\\
        =4H&\sqrt{T(K\adareg(T)+D)}+D,
    \end{align*}
    where $D = \tilde{O}(\sqrt{\beta_TH\dim_E(\cR,T^{-1})T}))$.
\end{theorem}
Comparing with Theorem~\ref{thm: incentive designer regret unknown reward}, we see that the steering gap has an additional term originating from the estimation of $U$. Furthermore, the bound of the number of epochs $K$ has an additional $T^{1/6}$. Similar to the discussion about $\cR$ in Section~\ref{appx:example_function_class}, we can also bound $\beta^U_K$ and $\dim_E(\cU,T^{-1})$ under suitable assumptions about $\cU$. If $\beta^U_K,\dim_E(\cU,T^{-1})\in\tilde\cO(1)$ and $\adareg(T)=\tilde\cO(\sqrt{T})$, both the steering cap and steering cost are in $\tilde\cO(T^{5/6})$ (ignoring all other constants).

\begin{proof}
We can adapt the proof of Theorem \ref{thm: incentive designer regret unknown reward} by choosing the following regret decomposition.
\begin{align*}
    U(\sad^{\pi^*})-U(\bmu^{t})
    =\left(U(\sad^{\pi^*})-\hat{U}^k(\hat\sad^{k}_*)\right)
    +\left(\hat{U}^k(\hat\sad^{k}_*)-\hat{U}^k(\bmu^{t})\right)
    +\left(\hat{U}^k(\bmu^{t})-U(\bmu^{t})\right)
\end{align*}
Using Lemma~\ref{lem: true unknown reward in confidence set} (and replacing $\cR$ by $\cU$ in the Lemma), we have $U\in\bigcap_{k=1}^K\hat\cU^k$ with probability at least $1-2\delta$. Thus, with probability at least $1-2\delta$, the first term can be bounded by 0 using optimism. The second term can be bounded in the same way as in the proof of Theorem \ref{thm: incentive designer regret unknown reward}. Summing over all $t$, the last term accumulates to
\begin{align*}
    \sum_{t=1}^T\hat{U}^k(\bmu^{t})-U(\bmu^{t})
    =\sum_{k=1}^K\sum_{t=T_{k-1}+1}^{T_k}\hat{U}^k(\bmu^{t})-U(\bmu^{t})
    \leq\sum_{k=1}^K\sum_{t=T_{k-1}+1}^{T_k}w_{\hat{\cU}^k}(\bmu^{t}).
\end{align*}
Using the fact that $T_k-T_{k-1}\leq T_{\mathit{epoch}}$ and Lemma \ref{lem: eluder confidence set bound} with $\epsilon=T^{-1}$, the sum above is at most
\begin{align*}
    \frac{KT_{\mathit{epoch}}}{T}+U_{\max}T_{\mathit{epoch}}\dim_E(\cU, T^{-1})+4\sqrt{\beta_T^UKT_{\mathit{epoch}}\dim_E(\cU,T^{-1})}.
\end{align*}
We also have to find a new bound for $K$. As before, we can enter the if block at most $HSA\log_2T$ times because of the first condition. In addition, we can enter the if block at most $T/T_{\mathit{epoch}}$ times due to the condition $T_k\geq T_{\mathit{epoch}}$. Therefore, $K\leq T/T_{\mathit{epoch}}+HSA\log_2T$ and $KT_{\mathit{epoch}}\leq T+HSAT_{\mathit{epoch}}\log_2T$.

Now, we set $T_{\mathit{epoch}}=T^{5/6}$. Then, $K\leq T^{1/6}+HSA\log_2T$ and
\begin{align*}
    \sum_{t=1}^T\hat{U}^k(\bmu^{t})-U(\bmu^{t})
    \leq\cO\left(T^{5/6}U_{\max}\dim_E(\cU,T^{-1})+\sqrt{\beta_T^U\dim_E(\cU,T^{-1})T}\right).
\end{align*}

Finally, the steering gap is the previous bound of Theorem~\ref{thm: incentive designer regret unknown reward} plus the above term.

With regard to the steering cost, the only change is the bound of $K$.

\end{proof}

%% file: Appendix/MainResultsSummary.tex
\section{SUMMARY OF MAIN RESULTS}\label{appx:summary}
In the following, we summarize the main theorems in this paper under Assump.~\ref{assump:realizability},~\ref{ass: adaptive regret} and~\ref{assump:Lipschitz}.
We study the steering gaps and costs of four settings. The settings are categorized depending on whether $M^*$ (or $\pi^* := \argmax_{\pi\in\Pi} U(\mu_{M^*}^\pi)$) is known or not, and whether the intrinsic reward function $r^*$ is zero or non-zero and unknown.
\begin{center}
\def\arraystretch{1.3}
\begin{tabular}{l l l l l l}
\hline
\textbf{Setting} & \textbf{$r^* = 0$?} & \textbf{Steering Gap} & \textbf{Steering Cost} & \textbf{Thm.}\\
\hline\hline
Known $M^*$ & \ding{51} & $\cO(L_U\sqrt{HSAT\adareg(T)})$ & $\cO(H\sqrt{T\adareg(T)})$ & \ref{thm: utility regret known model}\\ 
\makecell[tl]{Unknown $M^*$\\(known $\pi^*$)}  & \ding{51} & $\cO(L_U\sqrt{H^3SAT\adareg(T)})$ &
$\cO(H\sqrt{T\adareg(T)})$ & \ref{thm: regret of policy reward}\\ 
Unknown $M^*$ & \ding{51} & $\begin{aligned}{\mathcal{O}}(&L_U\sqrt{H^3SATK\adareg(T)}\\&+L_UH^3S\sqrt{AT\ln(THSA/\delta)})\end{aligned}$
& $\cO(H\sqrt{TK\adareg(T)})$ & \ref{thm: incentive designer regret}\\ 
Unknown $M^*$ & \ding{55} & $\begin{aligned}
    \cO(&L_U\sqrt{H^3SAT(K\adareg(T)+D)}\\
    &+L_UH^3S\sqrt{AT\ln(THSA/\delta)})
\end{aligned}$ & $\begin{aligned}\cO(H\sqrt{T(K\adareg(T)+D)}) \\ + D + C_T(\{\bmu^{t},r_{\max}\cdot \mathbf{1} - r^*\}_{t=1}^T)\end{aligned}$ & \ref{thm: incentive designer regret unknown reward} \\ 
\hline
\end{tabular}
\end{center}
Here $K = \cO(HSA\log T)$ and $D =\tilde{\cO}(\sqrt{\beta_TH\dim_E(\cR,T^{-1})T}))$, where $\dim_E$ is the eluder dimension of reward function class $\cR$, and $\beta_T = \tcO(1)$.

%% file: Appendix/Other_Related_Works.tex
\section{OTHER RELATED WORKS}\label{appx:related_works}

\paragraph{More Elaboration on Comparison between the Steering Setting and Contract Design Setting}
The steering setup differs from previous incentive design literature in two aspects: (1) it deals with “learning agents” continuously updating their policies and (2) it cares about the steering gap towards a target policy and the accumulative steering cost. One of the most related and representative existing problem setups is contract design (a.k.a. the principal-agent problem), which is a classical problem dating back to the seminal work \citep{holmstrom1979moral} in 1979. 
As we discussed in Sec.~\ref{sec:related_work}, it considers a similar mediator-agents interaction procedure.
In the following, we elaborate more on the comparison between those two settings to support our steering setting.

\begin{itemize}
    \item[(1)] Contract design assumes the agents respond optimally to the mediator/principal (e.g. maximize the total return including the incentives by mediator), which is a quite strong assumption and “simplifies” the problem by making the agents’ behaviors predictable. 
    
    In contrast, the steering framework treats the agents’ behavior as a dynamic process. For example, \citet{zhang_steering_2024} and ours consider no-regret behaviors, and \citep{huang_learning_2024,canyakmaz_steering_2024} assumes Markovian learning dynamics. Such a non-stationarity is more reasonable in practice and introduces additional challenges in achieving low the steering gap and cost.

    \item[(2)] Contract design considers a more challenging objective, and targets at finding the optimal incentive design to maximize the mediator’s gain deducted by the incentivizing cost. Usually, it also assumes the agents’ behaviors are unobservable. Due to such challenges, most of the contract design literature focuses on single-agent setting and assumes the knowledge of the model.

    On the other hand, the steering setting considers steering the agents to some target policies maximizing some utility function, which makes the framework more general. Besides, we do not pursue the optimality in steering cost but sub-linearity would be enough. This is reasonable because in many scenarios we only have budget constraints but do not have to achieve the optimum. Such a relaxation also makes the problem more tractable.
\end{itemize}

\paragraph{Mean-field game}
The mean-field game (MFG) is an important framework to model systems with a large number of symmetric agents \citep{lauriere_learning_2024}. 
Most works in the context of MFGs focus on learning equilibrium policies.
As the pioneers, \citet{lasry2007mean} and \citet{huang2006large} reveal that learning Nash Equilibrium (NE) is computationally efficient under monotonicity conditions if the model is known in advance.
Without the knowledge of the true model, many previous works contribute sample-efficient model-free \citep{guo_learning_2021,yardim_policy_2022,perolat_scaling_2021} and model-based \citep{huang_statistical_2023, huang_model-based_2024} methods to compute NE.
Our mean-field game definition is similar to the general MFG setting \citep{guo_learning_2021}, but unlike them, we assume transitions are density-independent and allow independence of agents' policies. This density-independent transition assumption has been frequently considered in previous works \citep{lasry2007mean, huang2006large,hu_mf-oml_2024,perolat_scaling_2021}. 
To our knowledge, we are the first to investigate steering agents' behaviors in the context of the mean-field game.

\paragraph{Mathematical Programming with Equilibrium Constraints (MPEC) and Mechanism Design}

MPEC considers a bilevel optimization formulation, where the upper level can be utility maximization problem and the lower level involves equilibrium constraints \citep{Luo1996Mathematical}.
There is a line of research works \citep{liu_inducing_2022,Wang2021Coordinating,Yang2022Adaptive} consider gradient-based approaches to solve MPEC problems. 
They usually require strong assumptions on computing hyper-gradients, which may fail to be satisfied in most games.
In contrast, we do not involve those assumptions or restrict the target policies are equilibria.
We only assume the agents are no-regret learners and do not require them to solve the equilibria induced by modified reward functions.

Another related field within game theory is Mechanism Design, which focuses on designing rules or systems (mechanisms) to achieve a specific objective, especially when participants (agents) have private information and act according to their own interests.
Most recent works consider mechanism design on Markov Games \citep{Curry2024Automateda,baumann2020adaptive}.

\citet{Guo2023MESOB} consider a bi-level optimization framework and another bi-objective variant, where the goal of the social planner is to solve an equilibrium policy maximizing some social welfare function.
They do not consider the usage of steering reward to intervene agents, and focus on the optimization side without considering model uncertainty.
In contrast, we study the incentive design problem, and focus on how to explore and design appropriate steering rewards to guide agents' behaviors without knowledge of the model.

\paragraph{Steering with advertisements}
Instead of steering agents with additional incentive payments, there are also settings which allow the modeling of advertisement campaigns \citep{balcan_circumventing_2013, balcan2011LeadingDynamicsToGoodBehavior, balcan2010CircumventingThePriceOfAnarchy} . The basic idea is that the mediator can advertise some strategy to the agents, who then choose between it and the best response. As opposed to steering with rewards (i.e., incentives), the mediator does not change the reward structure. The goal is to ``nudge'' the agents towards a better equilibrium.